\documentclass{article}

\usepackage{microtype}
\usepackage{graphicx}
\usepackage{booktabs} 
\usepackage{xcolor}
\usepackage{amsmath,amssymb,amsthm,amsfonts}
\usepackage{dsfont}
\usepackage{subcaption}

\newtheorem{theorem}{Theorem}
\newtheorem{lemma}{Lemma}
\newtheorem{definition}{Definition}
\newtheorem{corollary}{Corollary}
\newtheorem{proposition}{Proposition}
\usepackage{algorithm}
\usepackage{hyperref}

\usepackage[accepted]{icml2020}
\icmltitlerunning{Adversarial Learning Guarantees for Linear Hypotheses and Neural Networks}

\usepackage[utf8]{inputenc}
\usepackage{graphicx}
\usepackage{float}
\usepackage{caption}
\usepackage{subcaption}
\usepackage{multirow, multicol}

\usepackage{amsmath}
\usepackage{amsfonts}
\usepackage[nointegrals]{wasysym}
\usepackage{hyperref}
\usepackage{microtype}
\usepackage{nccmath}

\usepackage{cases}
\usepackage{mathtools}
\usepackage{amsmath}
\usepackage{amsthm}
\usepackage{amsfonts}
\usepackage{amssymb}
\usepackage{wrapfig}
\usepackage{graphicx}
\usepackage[mathscr]{euscript}

\usepackage{MnSymbol}
\DeclareMathAlphabet\mathbb{U}{msb}{m}{n}
\usepackage{xpatch}

\makeatletter
\newtheorem*{rep@theorem}{\rep@title}
\newcommand{\newreptheorem}[2]{%
\newenvironment{rep#1}[1]{%
\def\rep@title{#2 \ref{##1}}%
\begin{rep@theorem}}%
{\end{rep@theorem}}}
\makeatother

\def\Rset{\mathbb{R}}

\DeclareMathOperator{\spn}{span}
\DeclareMathOperator{\interior}{int}
\DeclareMathOperator{\dual}{dual}
\DeclareMathOperator*{\E}{\mathbb E}

\DeclareMathOperator*{\argmin}{argmin}

\DeclareMathOperator*{\sgn}{sgn}

\newcommand{\cA}{\mathcal{A}}
\newcommand{\cC}{\mathcal{C}}
\newcommand{\cD}{\mathcal{D}}
\newcommand{\cF}{\mathcal{F}}
\newcommand{\cG}{\mathcal{G}}

\newcommand{\cS}{\mathcal{S}}

\newcommand{\cZ}{\mathcal{Z}}

\newcommand{\bA}{{\mathbf A}}

\newcommand{\bI}{{\mathbf I}}

\newcommand{\bM}{{\mathbf M}}

\newcommand{\bU}{{\mathbf U}}

\newcommand{\bW}{{\mathbf W}}
\newcommand{\bX}{{\mathbf X}}

\newcommand{\bZ}{{\mathbf Z}}

\newcommand{\bb}{{\mathbf b}}
\newcommand{\bc}{{\mathbf c}}
\newcommand{\be}{{\mathbf e}}

\newcommand{\bs}{{\mathbf s}}
\newcommand{\bt}{{\mathbf t}}
\newcommand{\bu}{{\mathbf u}}
\newcommand{\bv}{{\mathbf v}}
\newcommand{\bw}{{\mathbf w}}
\newcommand{\bx}{{\mathbf x}}
\newcommand{\by}{{\mathbf y}}
\newcommand{\bz}{{\mathbf z}}
\newcommand{\bsig}{{\boldsymbol \sigma}}
\newcommand{\bsigma}{{\boldsymbol \sigma}}

\newcommand{\R}{\mathfrak R}

\newcommand{\h}{\widehat}

\newcommand{\wt}{\widetilde}
\newcommand{\e}{\epsilon}
\newcommand{\set}[2][]{#1 \{ #2 #1 \} }
\newcommand{\ignore}[1]{}

\hypersetup{
	colorlinks   = true,
	urlcolor = blue,
	linkcolor= blue,
	citecolor= blue
}

\newcommand{\eps}{\e}
\newcommand{\Ex}[2][]{\E_{#1}\left[ #2\right]}
\newcommand{\du}[1]{ {#1^*}}

\newcommand{\one}{\mathbf{1}}
\newcommand{\Esymb}{\mathbb{E}}

\newcommand{\norm}[1]{{\left|\left| #1 \right|\right|}}
\newcommand{\la}{\left \langle}
\newcommand{\ra}{\right \rangle}

\newif\ifnotes\notestrue

\newreptheorem{theorem}{Theorem}
\newreptheorem{lemma}{Lemma}

\ifnotes
\usepackage{color}
\definecolor{mygrey}{gray}{0.50}
\newcommand{\notename}[2]{{\textcolor{blue}{\footnotesize{\bf (#1:} {#2}{\bf ) }}}}

\else
\newcommand{\notename}[2]{{}}

\fi

\makeatletter
\newcommand*{\addFileDependency}[1]{
  \typeout{(#1)}
  \@addtofilelist{#1}
  \IfFileExists{#1}{}{\typeout{No file #1.}}
}
\makeatother

\begin{document}

\twocolumn[
\icmltitle{Adversarial Learning Guarantees for\\ 
Linear Hypotheses and Neural Networks}

\icmlsetsymbol{equal}{*}

\begin{icmlauthorlist}
\icmlauthor{Pranjal Awasthi}{google}
\icmlauthor{Natalie Frank}{nyu}
\icmlauthor{Mehryar Mohri}{googlenyu}
\end{icmlauthorlist}

\icmlaffiliation{google}{Google Research and Rutgers University}
\icmlaffiliation{nyu}{Courant Institute of Math. Sciences}
\icmlaffiliation{googlenyu}{Google Research and Courant Institute of Math. Sciences}

\icmlcorrespondingauthor{Natalie Frank}{nf1066@nyu.edu}

\vskip 0.3in
]
\printAffiliationsAndNotice{}

\begin{abstract}

  Adversarial or test time robustness measures the susceptibility of a
  classifier to perturbations to the test input. While there has been
  a flurry of recent work on designing defenses against such
  perturbations, the theory of adversarial robustness is not well
  understood. In order to make progress on this, we focus on the
  problem of understanding generalization in adversarial settings, via
  the lens of Rademacher complexity.

  We give upper and lower bounds for the adversarial empirical
  Rademacher complexity of linear hypotheses with adversarial
  perturbations measured in $l_r$-norm for an arbitrary $r \geq
  1$. This generalizes the recent result of Yin et
  al.~\cite{YinRamchandranBartlett2019} that studies the case of
  $r = \infty$, and provides a finer analysis of the dependence on the
  input dimensionality as compared to the recent work of Khim and
  Loh~\cite{khim2018adversarial} on linear hypothesis classes.  We
  then extend our analysis to provide Rademacher complexity lower and
  upper bounds for a single ReLU unit. Finally, we give adversarial
  Rademacher complexity bounds for feed-forward neural networks with
  one hidden layer. Unlike previous works we directly provide bounds
  on the adversarial Rademacher complexity of the given network, as
  opposed to a bound on a surrogate. A by-product of our analysis also
  leads to tighter bounds for the Rademacher complexity of linear
  hypotheses, for which we give a detailed analysis and present a
  comparison with existing bounds.
\end{abstract}

\section{Introduction}

Robustness is a key requirement when designing machine learning models
and comes in various forms such as robustness to training set
corruptions, missing feature values, and model
misspecification.
In recent years, requiring robustness to {\em adversarial} or {\em
  test time} perturbations has become a key requirement. Starting with
the work of \citet{SzegedyZaremba2014} it has now been well
established that deep neural networks trained via standard gradient
descent based algorithms are highly susceptible to imperceptible
corruptions to the input at test time \citep{goodfellow2014explaining,
  chen2017targeted, EykholtEvtimov2018, carlini2018audio}. This has
led to a proliferation of work aimed at designing classifiers robust
to such perturbations \citep{madry2017towards, gowal2018effectiveness,
  gowal2019alternative, schott2018towards} and works aimed at
designing more sophisticated attacks to break such classifiers
\citep{athalye2018obfuscated, carlini2017towards, sharma2017breaking}

While the above works have made significant progress in designing
practical defenses, theoretical aspects of adversarial robustness are
currently poorly understood unlike other notions of
training set corruptions that have been widely studied in both the
statistics and the computer science communities
\citep{huber2011robust, kearns1993learning,
  kearns1994toward}. Theoretical understanding of adversarial
robustness presents three main challenges. The first is a
computational one since even checking the robustness of a given model
at a given test input is an NP-hard problem \citep{AwasthiDV19}. This
has been explored in recent works that construct specific instances of
learning problems where standard non-robust learning can be done
efficiently, but learning a robust classifier becomes computationally
hard \citep{bubeck2018adversarial, bubeck2018adversarial2,
  nakkiran2019adversarial, degwekar2019computational}. The second
challenge concerns whether achieving adversarial robustness requires
one to compromise on standard accuracy. Recent works have shown
specific instance where this tradeoff is inherent
\citep{tsipras2018robustness, raghunathan2019adversarial}.

Finally, the third challenge, the main focus of this work, is the
question of what quantity governs generalization in adversarial
settings, and how generalization in adversarial settings compares to
its non-adversarial counterpart. The recent work of
\citet{schmidt2018adversarially} has shown, via specific
constructions, that in some scenarios achieving adversarial
generalization requires more data as compared to adversarial
generalization. Furthermore, the work of \citet{montasser2019vc} casts
a shadow of doubt on the use of classical quantities such as the
VC-dimension of explain generalization in adversarial settings.

However, generalization of function classes of infinite VC dimension (like SVMs with a Gaussian kernel) can be explained via margin based bounds. Characterizing the Rademacher complexity of the function class is essential in these estimates. In a similar vein, we believe that providing non-trivial bounds on the adversarial Rademacher complexity can help shed light on when generalization is possible in adversarial settings via similar margin based bounds. The difficulty is that current bounds on adversarial Rademacher complexity are too loose and vacuous in many settings. This is the barrier that we aim to overcome in this work.

In order to make progress on the mystery of adversarial
generalization, a recent line of work \citep{khim2018adversarial,
  YinRamchandranBartlett2019} aims to study the notion of Rademacher
complexity for various function classes in
the adversarial settings. Focusing mainly on the case of linear
models, these works aim to quantify the additional overhead in sample
complexity that is incurred when requiring adversarial
generalization. Extending the ideas to the case of more general neural
networks becomes more challenging and as a result there works instead
bound the Rademacher complexity in terms of the Rademacher complexity
of an appropriate surrogate. In this work we extend this line of work
along several directions.

\noindent \textbf{Our Contributions}.  We provide a general analysis
of the adversarial Rademacher complexity of linear models that holds
for perturbations measured in any $\ell_p$ norm. This extends the
prior work of Yin et al. \yrcite{YinRamchandranBartlett2019} that
applies only to $\ell_\infty$ adversarial perturbation and provided a
finer analysis of linear models as compared to the work of
\citet{khim2018adversarial}. 

As a consequence of our analysis, we
provide a sharp characterization of when the adversarial Rademacher
complexity suffers from an additional dimension dependent term as
compared to its non-adversarial counterpart. This has algorithmic
implications for designing appropriate regularizers for adversarial
learning of linear models. As an additional byproduct, we are able to
provide improved Rademacher complexity bounds for linear classifiers,
even in non-adversarial scenarios!

As a next step towards understanding neural networks, we then extend
our analysis to provide data dependent upper and lower bounds on the
adversarial Rademacher complexity of a single ReLU unit.

Finally, we provide upper bounds on the adversarial Rademacher
complexity of one hidden layer neural networks. As opposed to prior
works \citep{YinRamchandranBartlett2019, khim2018adversarial}, our
bounds directly apply to the original network as opposed to a
surrogate. Our bounds for neural networks come in two forms. We first
provide a general upper bound that applies to any neural network with
Lipschitz activations. This bound as a dependence on the underlying
dimensionality of the input data. Next, we provide a finer data
dependent upper bound that is related to the {$\epsilon$-adversarial
  growth function} of the data, a quantity we introduce in this work.

\noindent \textbf{Comparison with Prior Work} The works of Yin et
al.~\yrcite{YinRamchandranBartlett2019} and Khim and
Loh~\yrcite{khim2018adversarial} previously studied the adversarial
Rademacher complexity of linear classifiers and neural networks. Our
work adds to this line of research in multiple
ways. In~\cite{YinRamchandranBartlett2019} the authors analyze the
adversarial Rademacher complexity of linear models when perturbations
are measured in $\ell_\infty$ norm. They show that in this case the
adversarial Rademacher complexity of the loss class is bounded by the
sum of its non-adversarial counterpart and a dimension dependent
term. Our result is a strict generalization
of~\cite{YinRamchandranBartlett2019} because we provide the analysis
of adversarial Rademacher complexity when the perturbations are
measured in any general $\ell_r$ norm. 

The recent work of Khim and
Loh~\yrcite{khim2018adversarial} also studies the adversarial
Rademacher complexity of linear models under general
perturbations. While the bounds are qualitatively similar, our
analysis explicitly identifies the dimension dependent term in the
general case and as a result can be used to perform better model
selection when optimizing the adversarial loss. In addition, we
provide a matching lower bound on the adversarial Rademacher
complexity of linear models. In the process, we also improve upon the
existing classical analysis of (non-adversarial) Rademacher complexity
of linear models, which is of independent interest.

For the case of neural networks, both the works of Yin et
al.~\yrcite{YinRamchandranBartlett2019} and Khim and
Loh~\yrcite{khim2018adversarial} replace the adversarial loss defined
as $\min_{x': \|x'-x\| \leq \epsilon} \phi(yf(x'))$, by a surrogate
upper bound and analyze the resulting Rademacher complexity of the
surrogate. In the work of Yin et
al.~\yrcite{YinRamchandranBartlett2019} the surrogate is chosen to be
an upper bound on the adversarial loss based on a semi-definite
programming~(SDP) based relaxation. In the work
of~\cite{khim2018adversarial} the surrogate is based on the
adversarial loss of another neural network that is derived from the
original one via a tree based decomposition. In general, these bounds
on the surrogate might not lead to meaningful generalization bounds on
the original adversarial loss. We instead directly analyze the
Rademacher complexity of the adversarial loss.

\section{Notation and Preliminaries}

We will denote vectors as lowercase bold letters (e.g., $\bx$) and
matrices are uppercase bold (e.g., $\bX$). The all ones vector is
$\one$. H\"older conjugates are denoted by a star (e.g.,
$\du{r}$). For a matrix $\bM$, the $(p, q)$-\emph{group norm} is defined as
the
$\norm\bM_{p, q} =
\|(\|\bM_1\|_1,\ldots ,\|\bM_d\|_p )\|_q$, where the $\bM_i$s are
the columns of $\bM$.  We focus on binary classification over examples
in $\Rset^d$ and adversarial perturbations measured in $\ell_r$-norm
for $r \geq 1$.  Given a loss function $\ell\colon \Rset \to
[0,c]$, we define the loss of a hypothesis $f\colon
\Rset^d \to \Rset$ on a pair $(\bx, y) \in \Rset^d \times
\{+1,-1\}$ as $\ell_f(\bx, y) = \ell(yf(\bx))$.
As in the standard setting of classification, given a sample $\cS =
\{(\bx_1, y_1), (\bx_2, y_2), \dots, (\bx_m, y_m)\}$ drawn
i.i.d.\ from a distribution $\cD$ over $\Rset^d \times \{+1,-1\}$, we
define the empirical risk and the expected risk of a hypothesis $f$ as
\begin{align*}
R_\cS(f) &= \frac{1}{m} \sum_{i=1}^m \ell_f(\bx_i, y_i)\\
R(f) &= \E_\cS[R_\cS(f)] 
\end{align*}
Given an instance space $\cZ$, let $\cF$ be a class of functions from
$\cZ \to \Rset$. Given $\cS = (z_1, z_2, \dots, z_m) \subset
\cZ^m$, the empirical Rademacher complexity of the class $\cF$ is
defined to be
\begin{align}
\R_\cS(\cF)=\E_\bsig \left[ \sup_{f\in \cF}\frac 1 m \sum_{i=1}^m \sigma_i f(z_i) \right] \label{eq:erc_def}
\end{align}
where $\bsig$ is a vector of i.i.d.\ Rademacher random
variables. A tight characterization of the
uniform convergence of empirical risk to its expected value is given in terms of Rademacher complexity. Good bounds tend to come from \emph{margin bounds:}

\begin{theorem}\cite{MohriRostamizadehTalwalkar2018}\label{th:erc}
    Let $\cF$ be a family of functions and let $\ell $ be a loss function. Define $\ell_\cF=\{(\bx,y)\mapsto \ell(y f(\bx))\}$. Further, let $\cS$ be a sample, and let $\rho>0,\delta>0$.
    Define $\Phi_\rho(x)$ to be the $\rho$-margin loss:
    \[\Phi_\rho(x)=\min(1,\max(0,1-\frac x\rho))\] 
    and set
    \[\h R_{\cS,\rho}(f)=\frac 1m \sum_{i=1}^m \Phi_\rho (y_if(x_i))\]
    Then
    \[R(f)\leq \h R_{\cS,\rho}(f)+\frac 2 \rho \R_\cS(\ell_\cF)+3\sqrt{\frac{\log \frac 2\delta}{2m}}\]
    holds with probability st least $1-\delta$.

\end{theorem}

This theorem is significant because margin bounds can yield meaningful guarantees for rich classes even with infinite VC-dimension.

\noindent \textbf{Robust Classification.} We now extend the
definitions above to their adversarial counterparts. In the setting of
adversarially robust classification, the loss at $(\bx, y)$ is
measured in terms of the worst loss incurred over an adversarial
perturbation of $\bx$ within an $\|\cdot\|_r$ ball of a certain
radius. We will denote by $\epsilon$ the magnitude of the allowed
perturbations. Given $\epsilon > 0$, $r \geq 1$, a data point
$(\bx,y)$, a function $f\colon \Rset^d \to \Rset$, and a loss function $l\colon
\Rset \to [0,c]$ we define the adversarial loss of $f$ at $(\bx,y)$ as
\begin{align*}
\tilde{\ell}_f(\bx, y) = \sup_{\|\bx - \bx'\|_r \leq \epsilon}\ell(yf(\bx'))
\end{align*}
Similarly, we define the adversarial empirical risk
and the adversarial expected risk of a hypothesis $f$ for a sample $\cS$ as follows:
\begin{align*}
\wt{R}_\cS(f) &= \frac{1}{m} \sum_{i=1}^m \tilde{\ell}_f(\bx_i, y_i)\\
\wt{R}(f) &= \Esymb_\cS[\wt{R}_\cS(f)]
\end{align*}
With the above definitions, the following is an immediate application of Theorem~\ref{th:erc} above.
\begin{theorem}[Robust margin bounds]
\label{th:robust-erc}
Let $\epsilon \geq 0$, $\rho$ and $r \geq 1$. Let $\cF$ be a family of
functions from $\Rset^d$ to $\Rset$ and $\ell$ be a loss function taking
values in $[0, c]$. For any distribution $\cD$ over $\Rset^d \times
\{+1, -1\}$, given $\delta > 0$, and a sample $\cS=((\bx_1,y_1)\cdots (\bx_m,y_m))$ drawn i.i.d.\ from $\cD$, the
following holds with probability at least $1-\delta$: $\forall f \in
\cF$
\begin{align}
\label{eq:robust-uniform-convergence}
\wt{R}(f)\leq  \wt{R}_\cS(f)+\frac 2\rho \wt\R_\cS(\ell_{\cF})+3c\sqrt{\frac {\log \frac 2\delta}{2m}}.
\end{align}
\end{theorem}
Here, $\wt\R_\cS(\ell_{\cF})$ is the adversarial Rademacher complexity of the class $\ell_{\cF}$, and is defined by
\begin{align}
\label{eq:adversarial-Rademacher-complexity-loss-class}
\wt\R_\cS(\ell_{\cF}) = \E_\bsig \left[ \sup_{f\in \cF}\frac 1 m \sum_{i=1}^m \sigma_i \sup_{\|{\bx_i-\bx'_i}\|_r \leq \epsilon}\ell(y_i f(\bx'_i)) \right].
\end{align}
Throughout the paper, we will assume that the loss function $\ell$ is
non-increasing, a property satisfied by many common loss functions
including the hinge loss, logistic loss and the exponential loss. In
that case, as pointed out in~\cite{YinRamchandranBartlett2019}, the following equality holds:
\[
\sup_{\|{\bx_i-\bx'_i}\|_r \leq \epsilon}\ell(y_i f(\bx'_i)) 
= \ell \Big(\inf_{\|{\bx_i-\bx_i'}\|_r \leq \epsilon} y_i f(\bx'_i)\Big).
\]
Furthermore, when $\ell(\cdot)$ is $L$-Lipschitz, 
by Talagrand's contraction Lemma~\cite{LedouxTalagrand91}, we have
$\wt\R_\cS(\ell_{\cF}) \leq L\R_\cS(\wt \cF)$,
where $\wt \cF$ is the class defined as
\begin{align*}
\wt \cF 
= \big\{(\bx,y) \mapsto \inf_{\|{\bx-\bx'}\|_r \leq \epsilon} y f(\bx')\colon f \in \cF\big\}.
\end{align*}
Hence we get that
\begin{align}
\wt\R_\cS(\ell_{\cF})
& \leq L\R_\cS(\wt \cF) \nonumber\\
& = L \E_\bsig \left[ \sup_{f\in \cF}\frac 1 m \sum_{i=1}^m \sigma_i \inf_{\|{\bx_i-\bx'_i}\|_r \leq \epsilon}y_i f(\bx'_i) \right] \label{eq:adversarial-Rademacher-expr}.
\end{align}
Providing sharp bounds for \eqref{eq:adversarial-Rademacher-expr} for various function classes $\cF$ will be the central focus of this work.

\section{Adversarial Rademacher Complexity of Linear Hypotheses}
\label{sec:lin_rc}

In this section we provide a sharp characterization of the adversarial Rademacher complexity, as defined in \eqref{eq:adversarial-Rademacher-expr}, for linear function classes with bounded $p$-norm and with perturbations measured in any $r$-norm. Prior work~\cite{YinRamchandranBartlett2019} studied the case when the perturbations are measured in the $\ell_\infty$-norm. Our general analysis leads to a deeper understanding of the interplay between the complexity of the hypothesis classes~(measured in $p$-norm) and the perturbation set~(measured in $r$-norm), and how this dictates whether one can expect an additional dimension dependent penalty in the adversarial case over its non-adversarial counterpart. Furthermore, our analysis explicitly characterizes the dimension dependent term on which the adversarial Rademacher complexity depends on. This provides a finer analysis than the work of \citet{khim2018adversarial} and also has algorithmic implications.
Formally, we study the case when
\begin{align}
\cF_p = \{\bx \mapsto \la \bw, \bx \ra \colon \|{\bw}\|_p \leq W\}\label{eq:linear_function_class}
\end{align}

\subsection{Rademacher Complexity of Linear Hypotheses}

A crucial aspect of our analysis in the linear case is a more general
upper bound on the Rademacher complexity of $p$-norm bounded linear
function classes, in the non-adversarial case. We first state this
general bound as it will play an important role in later sections when
analyzing the adversarial Rademacher complexity of ReLU functions and
more general neural networks. 

\begin{theorem}
\label{th:linear_rc}
Let $\cF_p$ be the class of functions defined in
\eqref{eq:linear_function_class}. Then, given a sample
$\cS = \{(\bx_1, y_1), \dots, (\bx_m, y_m)\}$ we have

\begin{align}
\R_\cS(\cF_p)\leq
\begin{cases}
\frac W m\sqrt{{2\log(2d)}}\| {\bX^T}\|_{2,\infty} & \text{if $p=1$} \nonumber \\
\frac{\sqrt{2}W}{m} \bigg[\frac{\Gamma( \tfrac{\du p + 1}{2} )}{\sqrt{\pi}} \bigg]^{\frac{1}{\du p}} \|{\bX^T}\|_{2,\du p} & \text{if $1 <p \le 2$} \\ 
\frac{W}{m}\| \bX^T \|_{2, \du p} & \text{if  $p \ge 2$} \nonumber
\end{cases}
\end{align}
\end{theorem}
Here $\bX$ is the $d \times m$ matrix with the data points $\bx_i$ as columns. We make a few remarks about the theorem above and defer its proof to Appendix~\ref{app:new_linear_rc_proof}. 
Some well-known bounds on the Rademacher complexity of $\cF_p$ are

\ignore{
\begin{fleqn}
\begin{align*}
\R_\cS(\cF_p)\leq
\end{align*}
\end{fleqn}
}
\begin{align}
\R_\cS(\cF_p)\leq
\begin{cases}
W\sqrt{\frac{2\log(2d)}m}\| \bX\|_{\max} & \text{ if  $p=1$} \\ 
\frac W m \sqrt{\du p-1} \|{\bX}\|_{\du p,2} & \text{ if $1< p\leq 2$} \label{eq:previous_linear_rc}
\end{cases}
\end{align}
\ignore{
Recall that the $\| \cdot\|_{p_1,p_2}$ is the group norm defined as the $p_2$-norm of the $p_1$-norm of the columns of $\bX$:
\[\| \bX\|_{p_1,p_2}=\|(\|\bx_1\|_{p_1},\ldots,\|\bx_m\|_{p_1})\|_{p_2}\]
and
}

\begin{figure}[t]
\centering
\begin{tabular}{cc}
\includegraphics[scale=.25]{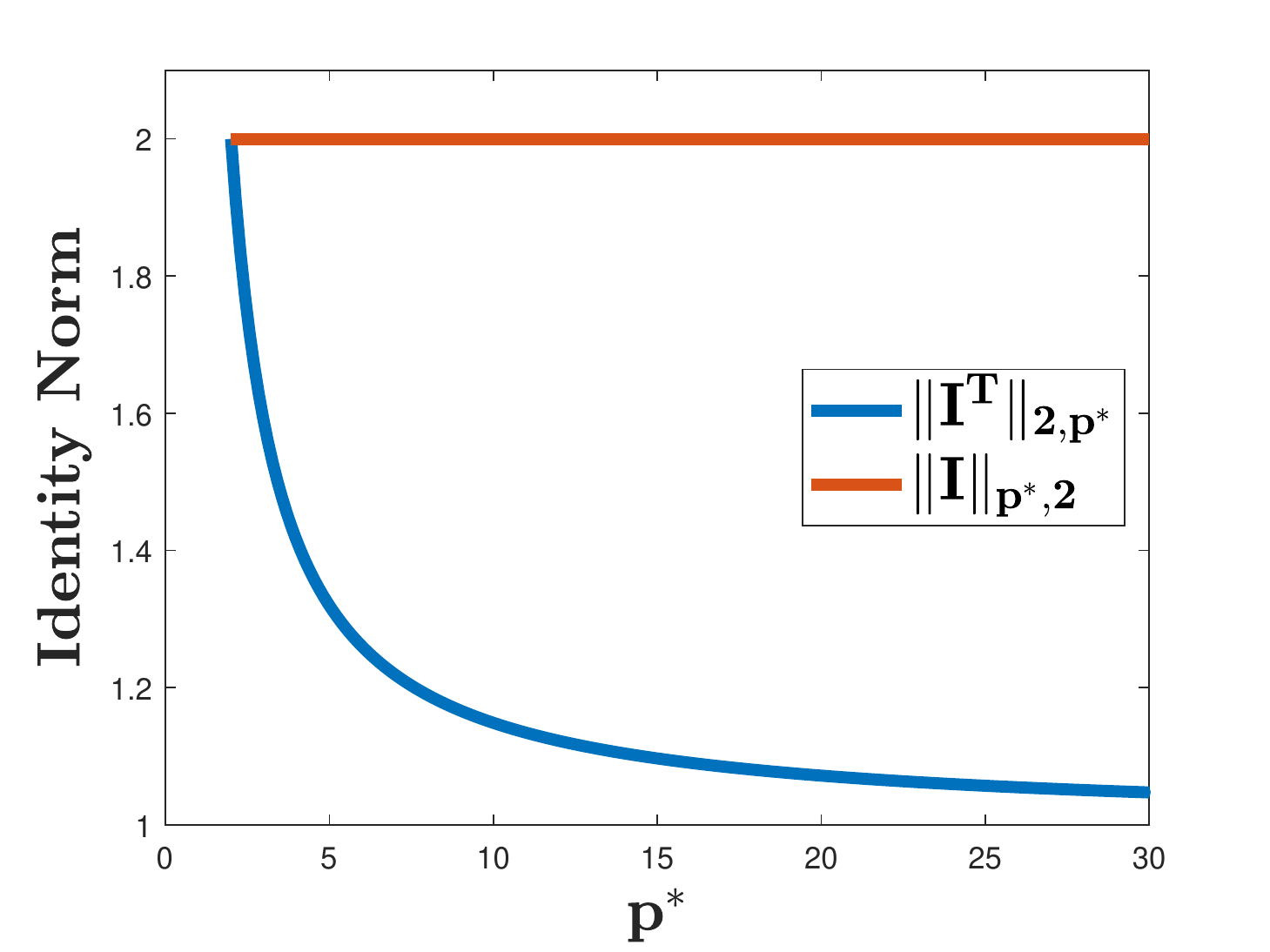}
& \includegraphics[scale=.25]{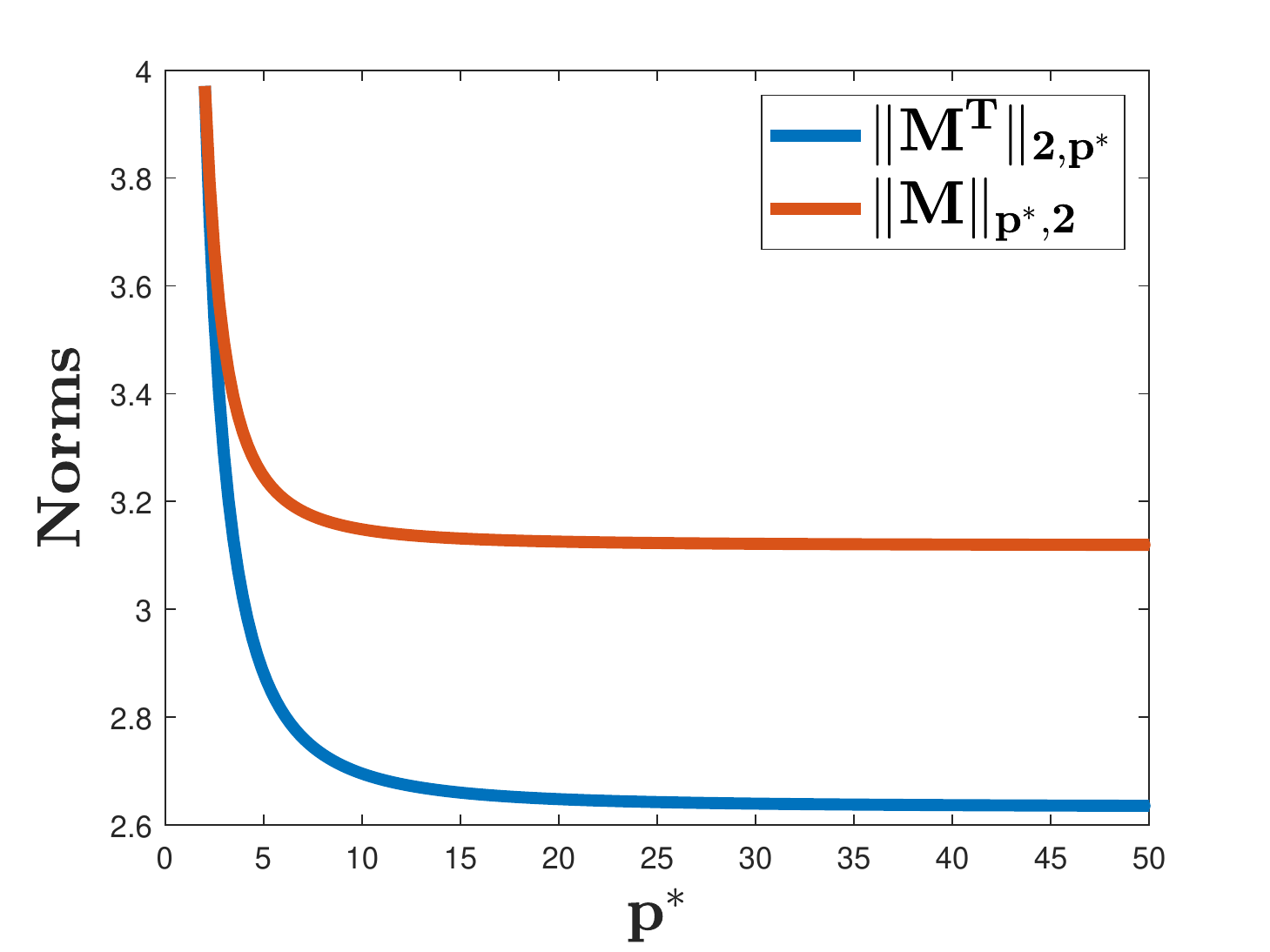}\\
(a) & (b)
\end{tabular}
\vskip -.1in
\caption{\label{fig:norms}(a) A plot comparing two norms of the
  $4 \times 4$ identity matrix, $\|\bI^T\|_{2,\du p}$ and
  $\|\bI\|_{\du p,2}$; the lower bound on the ratio of the two norms
  (\ref{eq:q_leq_p}) in Proposition~\ref{prop:norm_ratio} holds for
  this matrix. (b) Same as (a), but for Gaussian matrices.}
\vskip -.15in
\end{figure}

Although the case $p \in [1,2]$ in the theorem above is known
\cite{KakadeSridharantTewari2008, MohriRostamizadehTalwalkar2018}, we
provide a simpler proof of in
Appendix~\ref{app:prev_known_linear_rc_proof}. The inequality for
$p=1$ is further reproduced for completeness. Our new bound coincides
with (\ref{eq:previous_linear_rc}) when $p=2$ and is strictly better
otherwise. Readers familiar with Rademacher complexity bounds for
linear functions will notice that our bound in this case depends on
the norm $\|\bX^T\|_{2,p^*}$.  In contrast, standard bounds on the
Rademacher complexity of linear classes depend on $\|\bX\|_{p^*,
  2}$. In fact one can show that the $\|\bX^T\|_{2,p^*}$ is always
smaller than $\|\bX\|_{p^*, 2}$ for $p \in (1,2]$, that is
$p^* \geq 2$, as shown by the last inequality of \eqref{eq:q_leq_p} in
the following proposition.

\begin{proposition}
\label{prop:norm_ratio} 
Let $\bM$ be a $d\times m$ matrix.
If $q\leq p$, then 
\begin{equation}
\label{eq:q_leq_p}
\min(m,d)^{\frac 1p-\frac 1q} \|\bM^T\|_{p,q}\leq \|\bM\|_{q,p}\leq \|\bM^T\|_{p,q}\end{equation}
If $q\geq p$, then 
\begin{equation}\label{eq:p_leq_q}\min(m,d)^{\frac 1p-\frac 1q} \|\bM^T\|_{p,q}\geq \|\bM\|_{q,p}\geq \|\bM^T\|_{p,q}\end{equation}
These bounds are tight.

\end{proposition}
The proof is deferred to Appendix~\ref{app:norm_comparison}. To
visualize the ratio between these two norms, we plot the two norms for
various values of $\du p$ in Figure~\ref{fig:norms}.  For convenience, in the
discussion below, we set $c_1(p)=\sqrt{\du p-1}$ and
$c_2(p)=\sqrt 2\big[\frac{\Gamma( \tfrac{\du p + 1}{2} )}{\sqrt{\pi}}
\big]^{\frac{1}{\du p}}$. Regarding the growth of the
constant in our bound, one can show that as $\du p \to \infty$,
$c_2(p)$ grows asymptotically like $e^{-\frac 12}\sqrt {\du p} $.  In
fact one can show that
\[e^{-\frac 12}\sqrt{\du p}\leq c_2(p) \leq e^{-\frac 12} \sqrt{\du
    p+1}\] Furthermore, $c_2(p) \leq c_1(p)$ in the relevant
region~(see Appendix~\ref{app:constants}). In Figure~\ref{fig:constants} we plot $c_1(p),c_2(p)$ and
the bounds on $c_1(p)$ and $c_2(p)$ to illustrate the growth rate of these constants with $\du p$.
\begin{figure}[t]
\centering
\begin{tabular}{cc}
\includegraphics[scale=.35]{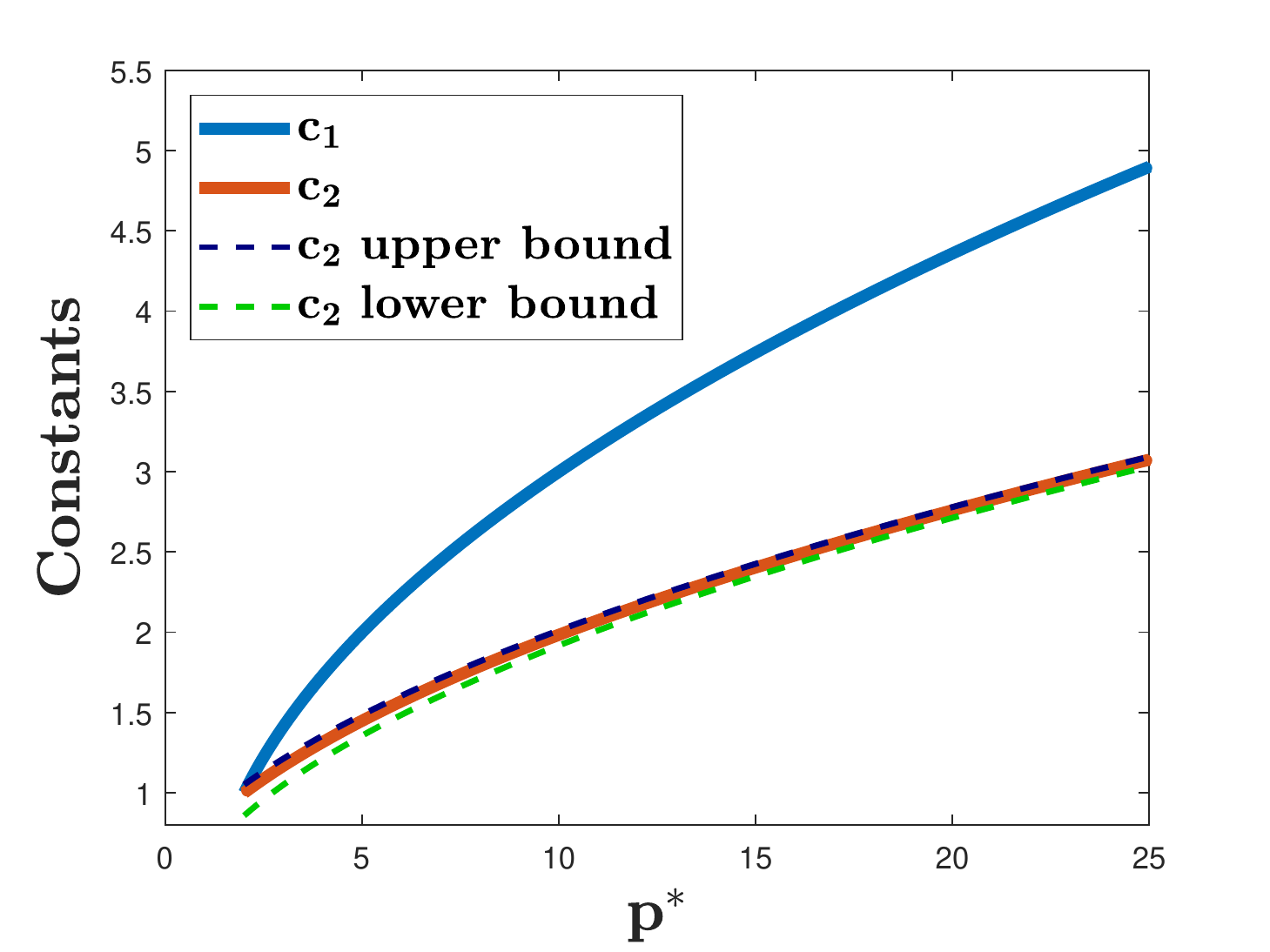}
\end{tabular}
\vskip -.15in
\caption{A plot of $c_1(p)$, $c_2(p)$, and the bounds from
  Lemma~\ref{lemma:constant-comparison}. Note that $c_1(2) = c_2(2)$ and
  that the upper and lower bounds on $c_2$ are tight.}
\vskip -.15in
\label{fig:constants}
\end{figure}
Proposition~\ref{prop:norm_ratio} and that $c_2(p) \leq c_1(p)$ imply
that the bounds we give for linear classes are stronger than what was
previously known.

\subsection{Adversarial Rademacher Complexity of Linear Hypotheses} 
We now extend our bounds from the previous section to provide a complete characterization
of the adversarial Rademacher complexity of linear function classes
under arbitrary $r$-norm perturbations. These theorems improve upon the recent work
of \citet{YinRamchandranBartlett2019} that studies $\infty$-norm perturbations and provide a finer analysis, with a matching lower bound, as compared to the recent work of \citet{khim2018adversarial}. Our main
result is stated below.

\begin{theorem}
\label{th:adversarial_linear}
Let $\epsilon > 0$, $p,r \geq 1$. Consider a sample $\cS = \{(\bx_1, y_1), \dots, (\bx_m, y_m)\}$ with $\bx_i\in \mathbb R^d$ and $y_i\in \{\pm 1\}$. Let $\cF_p$ be the class of
linear functions defined in \eqref{eq:linear_function_class}.
Then it holds that
\begin{align*}
\wt\R_\cS ( \cF_p)& \leq \bigg(\R_\cS(\mathcal {\cF}_p)+\e \frac{W}{2\sqrt m} \max(d^{1-\frac 1r-\frac 1p},1)\bigg)
\end{align*}
and
\begin{align*}
\wt\R_\cS ( \cF_p) &\geq \max
\bigg(\R_\cS(\cF_p), {W} 
\frac{\e \max(d^{1-\frac 1r-\frac 1p},1) }{2 \sqrt{2 m}}\bigg).
\end{align*}
\end{theorem}
Notice that when the perturbation is
measured in $\ell_\infty$-norm, i.e.\ $r = \infty$, we recover the
bound of \citet{YinRamchandranBartlett2019}. Hence the theorem above is a strict
generalization of the result of \citet{YinRamchandranBartlett2019}. Furthermore, when
$\epsilon = 0$, as expected, the adversarial Rademacher complexity
equals the standard Rademacher complexity of linear models and we can
use our improved bounds from Theorem~\ref{th:linear_rc}. The theorem
above has important implications for the design of regularizers in the
context of adversarial learning of linear models. As suggested by the
upper bounds above, if $1/r+1/p \geq 1$, then one can indeed perform
adversarially robust learning with minimal statistical overhead in
the standard classification setting! More specifically, in this case
the upper bound on the adversarial Rademacher complexity has at most
$W\epsilon/\sqrt{m}$ overhead on top of the standard bound from
Theorem~\ref{th:linear_rc} and is dimension independent. Noting that
$1-1/r = 1/r^*$, we get that for statistical efficiency one should
choose a $p$-norm regularizer on $\bw$, where $p \in [1,r^*]$. Our
lower bound on the other hand shows that any other choice of a $\ell_p$-norm
based regularizer will necessarily incur a dimension-dependent penalty.
\subsection{Proof sketch of Theorem~\ref{th:adversarial_linear}}\label{sec:adv_linear_proof}
We provide a brief sketch of the proof of Theorem~\ref{th:adversarial_linear} and provide the details in Appendix~\ref{sec:adv_linear}. As a first step, a simple argument shows that
$$
\inf_{\|\bx-\bx'\|_r\leq \e} y( \bw\cdot \bx')= y\bw\cdot \bx-\e
\|\bw\|_\du r.
$$
Using the above, we can write the adversarial Rademacher complexity as:
\begin{align}
&\R_\cS(\wt \cF_p) 
 = \E_\bsigma{ \left[ \sup_{\| \bw\|_p\leq W}  \la \bw ,\bu_\bsigma\ra -\e v_\bsigma \| \bw\| _{\du r}\right]}\label{eq:1}
\end{align} 
where, for convenience, we set $\bu_\bsigma=\frac 1m \sum_{i=1}^m y_i\sigma_i \bx_i$, $v_\bsigma=\frac 1m \sum_{i=1}^m \sigma_i$. 
Next, we present two key lemmas. 
\begin{lemma}
\label{lemma:norm_ratio} 
Let $1\leq p,r\leq \infty$ and let $d$ be the dimension. Then
\[
\sup_{\| \bw\|_p\leq 1}\| \bw\|_{\du r}=\max(1, d^{1-\frac 1r -\frac
  1p})
\]
\end{lemma}
\begin{lemma}
\label{lemma:big}
Let $v_\bsigma
=\frac 1m \sum_{i=1}^m \sigma_i$. Then it holds that 
\begin{align*} 
\E_\bsigma\left[{\sup_{\| \bw\|_p\leq W} v_\bsigma \| \bw\|_{\du r}}\right] \geq \frac{W\e \max(d^{1-\frac 1r-\frac 1p},1)}{{2 \sqrt{2m}} },
\end{align*}
and
\begin{align*}
\E_\bsigma\left[{\sup_{\| \bw\|_p\leq W} v_\bsigma \| \bw\|_{\du r}}\right] \leq \frac{W\e \max(d^{1-\frac 1r-\frac 1p},1)}{2 \sqrt m}
\end{align*}
\end{lemma}
For the upper bound, using the
sub-additivity of supremum and Lemma~\ref{lemma:big} yields
\begin{align*}
\R_\cS(\wt \cF_p)
&\leq \R_\cS(\mathcal F_p)  +\eps\Ex[\bsigma]{\sup_{\| \bw\|_p\leq W}  v_\bsigma \|
  \bw\|_{\du r}}\\
  &=\R_\cS(\mathcal {F}_p)+\frac 12 \eps \frac{W}{\sqrt
  m} \max(d^{1-\frac 1r-\frac 1p}, 1).
\end{align*} 
For the lower bound, we apply two symmetrization arguments and show
that
\begin{align}
\R_\cS(\wt \cF_p)
& = \Ex[\bsigma]{  \sup_{\| \bw\|_p\leq W}  -\la \bw ,\bu_\sigma\ra +\eps v_\bsigma \| \bw\| _{\du r}}\label{eq:2_main}\\
& = \Ex[\bsigma]  {\sup_{\| \bw\|_p\leq W}  \la \bw ,\bu_\bsigma\ra +\eps v_\bsigma \| \bw\| _{\du r}}. \label{eq:3_main}
\end{align} 
Averaging equations \eqref{eq:1} and \eqref{eq:3_main} and applying the
sub-additivity of supremum
gives: 
\begin{align*} 
\R_\cS(\wt \cF_p)
& =\frac 12\Ex[\bsigma]{  \sup_{\| \bw\|_p\leq W}\la \bw
  ,\bu_\bsigma\ra -\eps v_\bsigma \| \bw \|_{\du r}}\\
& \quad +\frac 12\Ex[\bsigma]{  \sup_{\| \bw\|_p\leq W}  \la \bw ,\bu_\bsigma\ra +\eps v_\bsigma \| \bw\| _{\du r}}\\
& \geq \Ex[\bsigma]{ \sup_{\| \bw\|_p\leq W} \la \bw,\bu_\bsigma\ra}=
  W\R_\cS({\mathcal F_p}).
\end{align*}
Now averaging \eqref{eq:2_main} and \eqref{eq:3_main}, applying
sub-additivity and Lemma~\ref{lemma:big}, the following holds:
\begin{align*}
\R_\cS(\wt \cF_p)
& =\frac 12 \Ex[\bsigma]{  \sup_{\| \bw\|_p\leq W}-\la \bw, \bu_\bsigma\ra +\eps v_\bsigma \| \bw \|_{\du r}}\\
& \quad +\frac 12\Ex[\bsigma]{  \sup_{\| \bw\|_p\leq W}  \la \bw ,\bu_\bsigma\ra +\eps v_\bsigma \| \bw \|_{\du r}}\\
& \geq \Ex[\bsigma]{\sup_{\|\bw\|_p\leq W} v_\bsigma \| \bw\|_\du r}\\&\geq \frac W {2\sqrt{2
  m}} \eps \max(d^{1-\frac 1p -\frac 1r}, 1).
\end{align*}

\section{Adversarial Rademacher Complexity of a Rectified Linear Unit}
\label{sec:relu}
As a first step towards providing a bound for neural networks, in this section we study the adversarial Rademacher complexity of linear functions composed with a rectified
linear unit (ReLU). Again we measure the size of functions in $p$
norm and define the function class by
\begin{equation}
\label{eq:relu_function_class}
\mspace{-8mu}
\mathcal G_p=\{ (\bx,y)\mapsto (y\la \bw,\bx\ra)_+\colon \|
\bw\|_p\leq W, y\in\{-1,1\}\} 
\end{equation}
where $(z)_+ = \max(z,0)$.
The following theorem presents a data-dependent upper bound on the adversarial Rademacher complexity of the ReLU unit.
\begin{theorem}
\label{th:data_dep_reLUupper}
Let $\cG_p$ be the class as defined in \eqref{eq:relu_function_class} and let $\cF_p$ be the corresponding linear class as defined in \eqref{eq:linear_function_class}. Then, given a sample $\cS= \{(\bx_1, y_1), \dots, (\bx_m, y_m)\}$, the adversarial Rademacher complexity of $\cG_p$
can be bounded as follows:
\begin{align*}
\wt \R_\cS(\cG_p)
& \leq {\mathfrak R }_{T_\eps}(\mathcal F_p)+\eps \frac{W}{2\sqrt{m}} \max(1,d^{1-\frac 1r-\frac 1p}),
\end{align*}
where $T_\eps=\{i\colon y_i=-1\text{ or }, y_i=1\text{ and }\|{\bx_i}\|_r> \eps\}$.
\end{theorem}
The second term in the bound above is similar to the dimension dependent term that appears in the linear case. The first term is the empirical Rademacher complexity of linear classes with bounded $p$-norm, but only measured on a carefully chosen subset of the
data. This implies that data points with positive labels that have small norm as compared to the perturbation $\epsilon$ do not affect the
Rademacher complexity. Hence, the guarantee in the theorem treats the two classes $+1$
and $-1$ asymmetrically. 

This phenomenon originates from a property of the function $(z)_+$. Recall that in our setup $(z)_+$ will later be
composed with a loss function $\ell$. Because $\ell(yz_+)$ is the
penalty incurred to the loss, the value $yz_+$ should be interpreted
as a margin. Since the function $\max(0,z)$ is always 0 for
$z\leq 0$, decreasing $z$ below 0 does not affect the the margin. On
the other hand increasing $z$ above zero will increase the margin. A
large margin for a point labeled $-1$ corresponds to making $z_+$ as
small as possible. As a result, every $z$ with $z\leq 0$ gives the
same margin. However, there is no upper bound on the margin for points
in the class $+1$. As a result, the classifier
$(y,\bx)\mapsto y(\la \bw,\bx\ra)_+$ treats all non-negative margins
for the class $-1$ in the same manner, but gives a higher reward for
larger margins for the class $+1$. 

This observation has implications
for adversarial classification; as shown in Appendix~\ref{app:relu_proofs}, an adversarially perturbed ReLU is
$y\max(\bw\cdot \bx-\eps y\| \bw\|_\du r,0)$. If $\bw\cdot \bx$ is
very negative, which corresponds to high confidence for $y=-1$, then
the perturbation would not change the value of the loss function. On
the other hand, if $\bw\cdot \bx$ were large and positive, a
perturbation would definitely change the value of the margin and then
influence the loss. We next complement our upper bound with a data dependent lower bound, stated below, on the adversarial Rademacher complexity.
\begin{theorem}
\label{th:data_dep_ReLUlower}
Let $\cG_p$ be the class as defined in \eqref{eq:relu_function_class}. Then it holds that
\[
\wt\R_\cS(\cG_p) \geq\frac W{2\sqrt 2 m} \sup_{\| \bs\|_p=1} \bigg(\sum_{i\in T_{\eps,\bs}}(\la \bs,\bx_i\ra-\eps y_i \| \bs\|_\du r)^2\bigg)^\frac 12
\]
where $T_{\eps,\bs} = \set{i\colon \la \bs,\bx_i\ra-y_i\eps \| \bs\|_\du r>0 }$.
\end{theorem}
A natural question that comes to mind is if one can characterize scenarios where the above lower bound leads to a dimension dependent term, as in the lower bound for linear hypotheses. In order to characterize this, for a given $s$ and $\delta > 0$, define the set $T^{\delta}_{\epsilon, \bs}$ as 
\[
T^{\delta}_{\epsilon, \bs} = \set{i\colon \la \bs,\bx_i\ra-(1+\delta y_i)y_i\eps \| \bs\|_\du r>0 }.
\]Notice that $T^{\delta}_{\eps, \bs}$ is a subset of $T_{\eps, \bs}$ and contains points in $T^{\delta}_{\eps, \bs}$ that have a non-trivial margin. Then we get that
\begin{align*}
    \R_\cS(\cG_p)&\geq\frac W{2\sqrt 2 m} \sup_{\| \bs\|_p=1}\bigg(\sum_{i\in T^{\delta}_{\eps,\bs}}(\la \bs,\bx_i\ra-\eps y_i \| \bs\|_\du r)^2\bigg)^\frac 12\\
    &\geq \frac W{2\sqrt 2 m} \sup_{\| \bs\|_p=1}\bigg(\sum_{i\in T^{\delta}_{\eps,\bs}}(\delta \eps \| \bs\|_\du r)^2\bigg)^\frac 12\\
    &= \frac{W \delta \epsilon}{2\sqrt 2 m} \sup_{\| \bs\|_p=1} |T^{\delta}_{\eps, \bs}| \|\bs\|_{\du r}.
\end{align*}
Denoting $\bs^*$ to be the vector that achieves the value $\sup_{\| \bs\|_p=1} \|\bs\|_{\du r}$ we get that
\begin{align*}
    \R_\cS(\cG_p) &\geq \frac{W \delta \epsilon}{2\sqrt 2 m} |T^{\delta}_{\eps, \bs^*}| \max(d^{1-\frac 1 p - \frac 1 r}, 1).
\end{align*}
Hence, if for a given constant $\delta > 0$, the size of the set $T^{\delta}_{\eps, \bs^*}$ is large then we expect a dimension dependent lower bound similar to the linear case.

\section{Adversarial Rademacher Complexity of Neural Nets}
Building on our analysis for the case of a single ReLU unit, we next give an upper bound on the adversarial Rademacher complexity
for the class of one-layer neural networks comprised of a
Lipschitz activation $\rho$ with $\rho(0)=0$. The guarantees of our theorem resemble the bound on the
standard Rademacher complexity of neural networks, as provided in
\cite{CortesGonzalvoKuznetsonMohri}. An analysis based on other forms of generalization bounds on neural nets is also possible, such
as that of \citet{BartlettFosterTelgarsky2017}. The family
of functions of such one-layer neural networks is defined
as follows:
\begin{align*}
\cG^n_p =
\set[\Big]{(\bx,y)\mapsto y\sum_{j=1}^n u_j \rho(\bw_j\cdot \bx)\colon \|\bu\|_1\leq \Lambda,  \|{\bw_j}\|_p\leq W}.
\end{align*}
Our main theorem is stated below.
\begin{theorem}
\label{th:one_layer_nn_rc} 
Let $\rho$ be a function with Lipschitz constant $L_\rho$ with $\rho(0)=0$ and consider perturbations in $r$-norm.
Then, the following upper bound holds for the adversarial Rademacher complexity of $\cG^n_p$:
\begin{align*}
\wt \R_\cS(\cG_p^n)
& \leq L_\rho \bigg[ \frac{W\Lambda \max(1,d^{1-\frac 1p-\frac 1r})(\|\bX\|_{r,\infty}+\e)}{\sqrt m} \bigg]
\times \\ 
& \quad \left( 1+\sqrt{d(n+1)\log(36)}\right).
\end{align*}
\end{theorem}
The proof is presented in Appendix~\ref{app:nn-lip}. The only
requirements on our activation function $\rho$ is that it is Lipschitz
and $\rho(0)=0$. This stipulation is satisfied by common activation
functions like the ReLU, the leaky ReLU, and the hyperbolic tangent,
but not the sigmoid or a step function. In comparison to the
adversarial Rademacher complexity of linear classifiers,
Theorem~\ref{th:one_layer_nn_rc} still includes a
$\max(1,d^{1-\frac 1r -\frac 1p})$ factor, again implying that one
should choose a model class with $p\leq \du r$. 
The complexity of the vector $\bu$ is bounded by $\ell_1$ norm as that is what turns out to be natural in the proof. However, the
dimension dependence is larger by a factor of $\sqrt d$. The dependence on the number of neurons $(\sqrt n)$ is also problematic. This fact is
unfortunate since a much larger sample size $m$ would be required for good
generalization. 
In the next section we present a promising approach towards removing the dependence on dimension and the number of neurons in the above bound.

\section{Towards Dimension-Independent Bounds}

In this section we introduce a new framework for analyzing the adversarial Rademacher complexity of neural networks with ReLU activations. Unlike the case of linear hypotheses, the dimension-dependent term in the upper bound in
Theorem~\ref{th:one_layer_nn_rc} 
cannot be avoided by simply picking the appropriate norm $p$. In particular, deriving dimension-independent bounds for the adversarial Rademacher complexity of neural networks is a difficult problem. Prior works \citep{YinRamchandranBartlett2019, khim2018adversarial} have resorted to bounding the adversarial Rademacher complexity of surrogates that are more tractable. However, it is not clear how those guarantees translate into meaningful bounds on the complexity of the original network. In this section, we present an approach towards obtaining dimension-independent bounds on the adversarial Rademacher complexity of the original network. 

A major component of the difficultly in analyzing adversarial Rademacher complexity relates to providing a tight characterization of the optimal adversarial perturbation for a given point $\bx_i$, i.e., 
\begin{align}
\label{eq:z-star}
\bs^*_i = \argmin_{\bs: \|\bs\|_r \leq 1} y_i \sum_{j=1}^n u_j(\bw_j \cdot (\bx_i +\e \bs))_\plus
\end{align}
Thus, to begin, we study properties of such adversarial perturbations to the neural network. Afterwards, we leverage these properties to bound the adversarial Rademacher complexity. Notably, the proofs of these properties heavily rely on the fact that the activation function is ReLU and not any other Lipschitz function. As in the previous section, we will focus on the family of a one layer-network $\cG^n_p$ with activation $\rho(z)=z_+$.

\subsection{Characterizing Adversarial Perturbations}

In this section, we discuss characteristics of adversarial
perturbations to neural networks with ReLU activations. The following theorem implies that,
if the perturbations are bounded in $\ell_r$-norm by $\e$, then typically the optimal
adversarial perturbations will have exactly $r$-norm $\e$.
\begin{theorem}
\label{th:unit_norm_at_opt} 
Let $d$ be the dimension and $n$ the number of neurons. Consider the problem
\begin{align}
\label{eq:objective}
& \inf_{\|\bs\|_r\leq 1} f(\bs) = \sum_{j = 1}^n
  u_j (\bw_j\cdot (\bx+\e \bs))_+ .
\end{align}
If either  $\|\bx\|_r\geq \e$ or $n<d$, an optimum is attained on the sphere $\set{\bs \colon \|\bs\|_r = 1}$.
Otherwise, an optimum is attained either at $\bs=-\frac 1\e \bx$ or on $\|\bs\|_r=1$.
\end{theorem}
The proof of the above theorem is deferred to Appendix~\ref{app:unit_norm_at_opt}. 
Theorem~\ref{th:unit_norm_at_opt} implies that if $n<d$, then the optimal perturbation always has norm $\e$. This result is significant because $n<d$ is a common scenario. At the same time, the theorem also implies that if $\|\bx\|_r\geq \e$, then the optimal perturbation still has norm $\e$. In practice, one expects the norm of the data points to be larger than the perturbation. Thus, on real world datasets, one would expect adversarial perturbations to always have norm $\e$.

For $1<r<\infty$, Theorem~\ref{th:unit_norm_at_opt} aids in finding a necessary condition for the optimum. This condition implies that critical points are characterized by specifying which $\bw_j$ satisfy $\bw_j\cdot(\bx+\e \bs)<0, \bw_j\cdot(\bx+\e \bs)=0$, and $\bw_j\cdot(\bx+\e \bs)>0$. The exact assertion is fairly involved, so we delay the statement of this theorem to Appendix~\ref{app:necessary_condition}. However, the theorem simplifies considerably for $r=2$ and we include this case below.

\begin{theorem}\label{th:opt_char_r_2} Assume that $\|\bx\|_r\geq \e$. Let $1<r<\infty$ and take $f$ as in Theorem~\ref{th:unit_norm_at_opt} and $s^*$ as the minimizer of \eqref{eq:objective}. Define the following three sets:
\begin{align*}
    N =\{j\colon \bw_j\cdot (\bx+\e \bs^*)<0\}\\
    Z=\{j\colon \bw_j\cdot (\bx+\e \bs^*)=0\}\\
    P=\{j\colon \bw_j\cdot (\bx+\e \bs^*)>0\}.
\end{align*}
$\bs^*$ is characterized by specifying the sets $N,Z$, and $P$.
Furthermore, if $r=2$, $\bs^*$ can be explicitly expressed in terms of these sets.
Let $P_Z$ be the projection onto $\spn\{\bw_j\}_{j\in Z}$ and $P_{Z^C}$ the projection onto the complement of this subspace.
Then, $\bs^*$ is given by
\[
\bs^*=-\left(\sqrt{1-\frac{\|P_Z\bx\|_2^2}{\e^2}}\frac{P_{Z^C}\sum_{j\in P} u_j \bw_j}{\left\| P_{Z^C}\sum_{j\in P} u_j \bw_j\right\|_2}+\frac {1}{\e} {P_Z\bx}\right).
\]
\end{theorem}

\subsection{Dimension-Independent Bound for ReLU Neural Networks}
\label{sec:rc_relu_nn}
\label{sec:dim-independent-neural-nets}

Observe that, given $\bu$ and the weight matrix $\bW$ with columns
$(\bw_1, \dots, \bw_n)$, each $\bx_i$ partitions these vectors into
three sets depending on whether at the optimal $\bs^*_i$,
$\bw_j \cdot (\bx_i + \e\bs^*_i)$ is positive, zero or
negative. As a result, given $\bW$ and $\bu$, the points in the data
set can be partitioned into sets depending on whether they induce the
same sign pattern on the columns of $\bW$. Let $\mathcal{C}_\cS$
denote the set of all such possible partitions and let
$\mathcal{C}^*_\cS$ be the size of this set. Indexing a particular
partition in this set by $\mathcal{C}$, let $n_\mathcal{C}$ be the
number of parts in this partition and define
$\Pi^*_\cS = \max_{\mathcal{C}} n_{\mathcal{C}}$. Notice that both
$\Pi^*_\cS$ and $\mathcal{C}^*_\cS$ are data-dependent quantities. We
next state a general theorem that does not explicitly depend on the
dimension and instead bounds the adversarial Rademacher complexity in
terms of the above data-dependent quantities.
\begin{theorem}
\label{th:one_layer_nn_rc_shatter} Consider the family of functions $\cG^n_p$ with activation function $\rho(z)=(z)_+$.
and perturbations in $r$-norm for $1<r<\infty$. Assume that for our sample $\|\bx_i\|_r\geq \e$.
Then, the following upper bound on the Rademacher complexity holds:
\[
\wt \R_\cS(\cG_p^n)
\!\leq\! \bigg[\! \frac{W\Lambda \max(1,d^{1-\frac 1p-\frac 1r})(\|\bX\|_{p^*,\infty}+\e)}{\sqrt m} \!\bigg] C^*_\cS \sqrt{\Pi^*_\cS}.
\]
\end{theorem}
Notice that the main difference between the above guarantee and the
one from the previous section is that the dimension-dependent term
$(1+\sqrt{d(n+1)\log(9 m)})$ has been replaced by data-dependent
quantities. Next, we discuss how to bound these data-dependent
quantities in terms of a notion of \emph{adversarial shattering} that
we introduce in this work.

\noindent \textbf{Bounding $\Pi^*_\cS$ and $\epsilon$-adversarial
  shattering.} A key quantity of interest in understanding the bounds
from the above theorem is $\Pi^*_\cS$. Notice that this corresponds to
the maximum number of partitions of the vectors $\bw_1, \dots, \bw_j$
that can be induced by the dataset $(\bx_1, \ldots,
\bx_m)$. Viewing the $\bw_j$s as examples and the $\bx_i$s as
hyperplanes, this corresponds to the number of sign patterns on $\bW$
that can be induced by $\cS$. In standard settings, this would be
bounded by the VC-dimension ($d$ in this case). However, we know more
about how the $\bx_i$s act on these vectors. Notice that at the
optimal $\bs^*_i$ for a given $\bx_i$, for some subset of vectors
$\bw_j \cdot \bx_i + \bw_j \cdot \e\bs^*_i \geq 0$, and for the rest
it must be that $\bw_j \cdot \bx_i + \bw_j \cdot \e\bs^*_i \leq
0$. Hence, not only does $\bx_i$ induce a sign pattern on the
$\bw_j$s, it does so with a certain margin. This is reminiscent of the
classical notion of {\em fat shattering}
\cite{MohriRostamizadehTalwalkar2018} from statistical learning
theory. However, in this case, the margin induced could itself depend
on the $\bw_j$s in a complex manner via the product of
$\bw_j \cdot \bs^*_i$. To formalize this intuition, we define the
following notion of $\epsilon$-adversarial shattering.

\begin{definition}
Fix the sample $\cS = ((\bx_1, y_1)\ldots (\bx_m, y_m))$ and $(\bw_1, \ldots, \bw_n)$.
Let $\bs_i = \argmin_{\|\bs\|_r\leq 1}y_i\sum_{j=1}^n
u_j(\bw_j\cdot (\bx_i+\e \bs))_+$, and define the following three sets:
\begin{align*}
    P_i=\{j\colon \bw_j\cdot (\bx+\e \bs_i)>0\}\\
    Z_i=\{j\colon \bw_j\cdot (\bx+\e \bs_i)=0\}\\
    N_i=\{j\colon \bw_j\cdot (\bx+\e \bs_i)<0\}.
\end{align*}
Let $\Pi_\cS(\bW)$ be the number of distinct $(P_i,Z_i,N_i)$s that are
induced by $\cS$, where $\bW$ is a matrix that admits the
$\bw_j$s as columns. We call $\Pi_\cS(\bW)$ the \emph{$\e$-adversarial
  growth function}. We say that $\bW$ is \emph{$\e$-adversarially
  shattered} if every $P\subset [n]$ is possible.
\end{definition}
Under certain assumptions, by carefully studying the above notion of
adversarial shattering one can obtain bounds of the form
$O(\frac{1}{\epsilon^2})$ on the maximum number of $\bw_j$s that can
be adversarially shattered by $\cS$. This lets us use an argument
similar in spirit to Sauer's lemma \cite{sauer1972density,
  shelah1972combinatorial} to bound $\Pi^*_\cS$ by
$n^{O(1/\epsilon^2)}$, thereby leading to a meaningful bound in
Theorem~\ref{th:one_layer_nn_rc_shatter}. We believe that a further
study of the above notion of adversarial shattering is the key to
proving general dimension-independent bounds on the adversarial
Rademacher complexity of neural networks.

\section{Conclusion}
In this work we presented a detailed study of the generalization properties of linear models and neural networks under adversarial perturbations. Our bounds for the linear case improve upon prior work and also lead to a novel analysis of the Rademacher complexity of linear hypotheses in non-adversarial settings as well. 
For the case of a single ReLU unit, while we have upper and lower bounds, it would be interesting to investigate the extent to which they are close to each other. Our analysis for the linear and ReLU hypotheses reveals that by choosing the appropriate norm regularization~($\ell_p$) on the weight matrices, one can indeed avoid dimension dependence and achieve generalization in adversarial settings with negligible statistical overhead as compared to the corresponding non-adversarial setting. Our analysis illustrates the importance of choosing $p$ satisfying $\frac 1r+\frac 1p\geq 1$ in algorithms. This relationship further suggests that for robustness to perturbations in an arbitrary norm 
$\|{\cdot}\|$, one could regularize by the dual norm of $\|{\cdot}\|$. Investigating this relationship could be future work.
Finally, it would be interesting to use our approach from Section~\ref{sec:rc_relu_nn} based on $\epsilon$-adversarial shattering to provide dimension-independent upper bounds on the adversarial Rademacher complexity of neural networks.


\bibliographystyle{icml2020}
\bibliography{adv}

\clearpage

\appendix
\onecolumn
\section{The Rademacher Complexity of Linear Classes [Proof of
  Theorem~\ref{th:linear_rc}]}
\label{app:linear_rc}

In this section, we provide a proof of Theorem~\ref{th:linear_rc} and
present improved bounds for the Rademacher complexity of linear
hypotheses. We will analyze each of the three sub-cases namely,
$p \in (1, 2]$, $p > 1$, and $p = 1$ separately in the subsections
that follow.  Recall that the group norm $\| \cdot\|_{p_1, p_2}$ of
matrix $\bX$ is defined by
\[
\| \bX \|_{p_1, p_2} = \|(\|\bx_1\|_{p_1}, \cdots, \|\bx_m\|_{p_1})\|_{p_2},
\]
where $\bx_1, \ldots, \bx_m$ are the columns of $\bX$. For
$p_1, p_2 \leq \infty$, this group-norm can be rewritten as follows:
\[
\| \bX \|_{p_1,p_2}
= \left[ \sum_{i=1}^m\left( \sum_{j=1}^d|X_{j,i}|^{p_1}\right)^\frac {p_2}{p_1} \right]^\frac 1{p_2}.
\]

\subsection{Case $p \in (1, 2]$}
\label{app:prev_known_linear_rc_proof}

For convenience,
we will use the shorthand $\bu_\bsigma = \sum_{i = 1}^m
\sigma_i \bx_i$. By definition of the dual norm,
we can write:
\begin{align*}
\R_S(\cF_p)
& = \frac{1}{m} \E_\bsigma \Bigg[ \sup_{\|{ \bw }\|_p \leq W } \bw \cdot \sum_{i = 1}^m
\sigma_i \bx_i \Bigg]\\
& = \frac{W}{m} \E_\bsigma \big[ \| \bu_\bsigma \|_\du p \big] & (\text{dual norm property})\\
& \leq \frac{W}{m} \sqrt{\E_\bsigma \big[ \| \bu_\bsigma \|_\du p^2 \big]}. &
(\text{Jensen's inequality})
\end{align*}

Now, for $\du p \geq 2$,
$\Psi\colon \bu \mapsto \frac{1}{2} \| \bu \|_\du p^2$ is
$(\du p - 1)$-smooth with respect to $\| \cdot \|_{p^*}$, that is, the
following inequality holds for all $\bx, \by \in \Rset^d$:
\[
\Psi(\mathbf y)\leq \Psi(\mathbf x) + \nabla \Psi(\bx)^T(\mathbf y-\bx)+\frac{\du p-1}2 \|{\mathbf y-\bx}\|^2_{\du p}
\] 
In view of that, by successively applying the $(\du p - 1)$-smoothness
inequality, we can write:
\begin{align*}
2 \Psi(\bu_\sigma) 
\leq 2 \sum_{k = 1}^m \bigg\langle \nabla \Psi \Big( \sum_{i = 1}^{k - 1}
\sigma_i \bx_i \Big) , \sigma_k \bx_k \bigg\rangle + (\du p - 1) \sum_{i =
1}^m \| \sigma_i \bx_i \|_\du p^2.
\end{align*}
Conditioning on $\sigma_1, \ldots, \sigma_{k - 1}$ and taking
expectation gives:
\begin{align*}
2 \E_\bsigma[\Psi(\bu_\sigma) ]
\leq (\du p - 1) \sum_{i =
1}^m \| \bx_i \|_\du p^2.
\end{align*}
Thus, the following upper bound holds for the empirical Rademacher
complexity:
\begin{align*}
\R_S(\cF_p) 
\leq \frac{W}{m} \sqrt{(\du p - 1) \sum_{i =
1}^m \| \bx_i \|_\du p^2}.
\end{align*}


\subsection{General case $p > 1$}
\label{app:new_linear_rc_proof}

Here again, we use the shorthand
$\bu_\bsig=\sum_{i = 1}^m \sigma_i\bx_i$. By definition of the dual
norm, we can write:
\begin{align*}
\R_S(\cF_p)
& = \frac{1}{m} \E_\bsigma \Bigg[ \sup_{\| \bw \|_p \leq W } \bw \cdot \sum_{i = 1}^m
\sigma_i \bx_i \Bigg]\\
& = \frac{W}{m} \E_\bsigma \big[ \| \bu_\bsigma \|_\du p \big] & (\text{dual norm property})\\
& \leq \frac{W}{m} \Big[ \E_\bsigma \big[ \| \bu_\bsigma \|_\du p^\du p \big]\Big]^{\frac{1}{\du p}}. &
(\text{Jensen's inequality, $p^* \in [1, +\infty)$})\\
& = \frac{W}{m} \Big[ \sum_{j = 1}^d
\E_\bsigma \big[ |\bu_{\bsigma, j}|^{\du p} \big] \Big]^{\frac{1}{\du p}}.
\end{align*}
Next, by Khintchine's inequality \cite{Haagerup1981}, the following
holds:
\begin{align*}
\E_\bsigma \big[ |\bu_{\bsigma, j}|^{\du p} \big] &\leq B_{\du p} \Big[ \sum_{i = 1}^m
x_{i, j}^2 \Big]^{\frac{\du p}{2}},
\end{align*}
where $B_{\du p} = 1$ for $p^* \in [1, 2]$ and
\begin{align*}
B_{\du p} &= 2^{\frac{\du p}{2}}\frac{\Gamma \big( \frac{\du p + 1}{2} \big)}{\sqrt{\pi}},
\end{align*}
for $p \in [2, +\infty)$.  This yields the following bound on the
Rademacher complexity:
\[
\R_S(\cF_p) \leq 
\begin{cases}
\frac{W}{m}\| \bX^T \|_{2, \du p} & \text{if }\du p \in [1, 2], \\[.25cm]
\frac{\sqrt{2}W}{m} \bigg[\frac{\Gamma \big( \tfrac{\du p + 1}{2} \big)}{\sqrt{\pi}} \bigg]^{\frac{1}{\du p}} \| \bX^T \|_{2, \du p} & \text{if } \du p \in [2, +\infty).
\end{cases}
\]

\subsection{Case $p = 1$}
\label{app:linear_p_1}

The bound on the Rademacher complexity for $p=1$ was previously known
but we reproduce the proof of this theorem for completeness. We
closely follow the proof given in
\cite{MohriRostamizadehTalwalkar2018}.
\begin{proof}
For any $i \in [m]$, $x_{ij}$ denotes the $j$th component of
$\bx_i$.
\begin{align*}
\R_\cS(\cF_1)  
& = \frac{1}{m} \E_\bsigma \left[ \sup_{\| \bw \|_1 \leq W}
\bw \cdot \sum_{i = 1}^m \sigma_i \bx_i  \right]\\
& = \frac{W}{m} \E_\bsigma \left[ \Big\| \sum_{i = 1}^m
\sigma_i \bx_i  \Big\|_\infty \right] & \text{(by definition of the
dual norm)}\\
& = \frac{W}{m} \E_\bsigma \left[ \max_{j \in [d]} \left| \sum_{i = 1}^m
\sigma_i x_{ij}  \right|  \right] & \text{(by definition of $\| \cdot \|_\infty$)}\\
& = \frac{W}{m} \E_\bsigma \left[ \max_{j \in [d]} \max_{s \in
\set{-1, +1}} s \sum_{i = 1}^m
\sigma_i x_{ij}   \right] & \text{(by definition of $| \cdot |$)}\\
& = \frac{W}{m} \E_\bsigma \left[ \sup_{\bz \in \cA} \sum_{i = 1}^m
\sigma_i z_i  \right] ,
\end{align*}
where $\cA$ denotes the set of $d$ vectors
        $\set{s (x_{1j}, \ldots, x_{mj})^\top \colon j \in [d], s \in
          \set{-1, +1}}$.  For any $\bz \in A$, we have
        $\| \bz \|_2 \leq \sup_{\bz\in A}
        \|\bz\|_2=\|\bX^T\|_{2,\infty}$. Further, $\cA$ contains
        at most $2d$ elements. Thus, by Massart's lemma
        \citep{MohriRostamizadehTalwalkar2018}, 
\begin{align*}
\R_\cS(\cF_1)  
& \leq W  \|\bX^T\|_{2, \infty} \frac{\sqrt{2 \log (2d)}}{m},
\end{align*}
which concludes the proof.
\end{proof}

\subsection{Comparing $\|\bM^T\|_{p, q}$ and $\|\bM\|_{q, p}$ [Proof
	of Proposition~\ref{prop:norm_ratio}]}
\label{app:norm_comparison}

In this section, we prove
Proposition~\ref{prop:norm_ratio}. This proposition implies that for
$p \in (1, 2)$, the group norm $\|\bX^\top\|_{2, p^*}$, is always a lower
bound on the term $\|\bX\|_{p^*,2}$. These two norms are a major component of the Rademacher complexity of linear classes. 
\begin{proof} First, \eqref{eq:p_leq_q} follows from \eqref{eq:q_leq_p} by substituting $\bM=\bA^T$ for a matrix $\bA$:
For $q\leq p$,
\[\min(m,d)^{\frac 1p -\frac 1q} \|\bA\|_{p,q}\leq \|\bA^T\|_{q,p}\leq \|\bA\|_{p,q}\] which implies that 
\[\|\bA^T\|_{q,p}\leq \|\bA\|_{p,q}\leq \min(m,d)^{\frac 1q-\frac 1p}\|\bA^T\|_{q,p}\]
However, now $p$ and $q$ are swapped in comparison to \eqref{eq:p_leq_q}. Now after swapping them again, for $p\leq q$,
\[\|\bA^T\|_{p,q}\leq \|\bA\|_{q,p}\leq \min(m,d)^{\frac 1p-\frac 1q}\|\bA^T\|_{p,q}\]
The rest of this proof will be devoted to showing \eqref{eq:q_leq_p}.
	
	Next, if $p=q$, then $\| \bM\|_{q,p}=\| {\bM^T}\|_{p,q}$.
	For the rest of the proof, we will assume that $q < p$. Specifically, $q < +\infty$ which allows us to consider fractions like $\frac pq$.  
	
	We will show that for $q < p$, the following inequality holds:
	$\|\bM\|_{q, p} \leq \|\bM^T\|_{p, q}$, or equivalently,
	$\|\bM\|^q_{q, p} \leq \|\bM^T\|^q_{p, q}$.
	
	We will use the shorthand $\alpha = \tfrac{p}{q} > 1$.
	By definition of the group norm and using the notation $\bU_{ij} =
	|\bM_{ij}|^q$, we can write
	\begin{align*}
	\|\bM\|^q_{q, p}
	= \bigg[ \sum_{j = 1}^d \Big[\sum_{i = 1}^m |\bM_{ij}|^q
	\Big]^{\frac{p}{q}} \bigg]^{\frac{q}{p}}
	&= \bigg[ \sum_{j = 1}^d \Big[\sum_{i = 1}^m \bU_{ij}
	\Big]^{\alpha} \bigg]^{\frac 1 \alpha}
	= \left\| 
	\left[
	\begin{smallmatrix}
	\sum_{i = 1}^m \bU_{i1}\\
	\vdots\\
	\sum_{i = 1}^m \bU_{id}
	\end{smallmatrix}
	\right]
	\right\|_{\alpha}\\
	& = \sup_{\| \bb \|_{\alpha^* \leq 1}}
	\left[
	\begin{smallmatrix}
	\sum_{i = 1}^m \bU_{i1}\\
	\vdots\\
	\sum_{i = 1}^m \bU_{id}
	\end{smallmatrix}
	\right] \cdot
	\left[
	\begin{smallmatrix}
	b_{1}\\
	\vdots\\
	b_{d}
	\end{smallmatrix}
	\right] & \text{(by def. of dual norm)}\\
	& \leq \sum_{i = 1}^m 
	\sup_{\| \bb \|_{\alpha^* \leq 1}}
	\left[
	\begin{smallmatrix}
	\bU_{i1}\\
	\vdots\\
	\bU_{id}
	\end{smallmatrix}
	\right] \cdot
	\left[
	\begin{smallmatrix}
	b_{1}\\
	\vdots\\
	b_{d}
	\end{smallmatrix}
	\right]  & \text{(sub-additivity of $\sup$)}\\
	& = \sum_{i = 1}^m \left\| \left[
	\begin{smallmatrix}
	\bU_{i1}\\
	\vdots\\
	\bU_{id}
	\end{smallmatrix}
	\right] \right\|_\alpha  & \text{(by def. of dual norm)}\\
	& = \sum_{i = 1}^m \Big[ \sum_{j = 1}^d |\bM_{ij}|^p
	\Big]^{\frac{q}{p}}
	= \|\bM^T\|^q_{p, q}.
	\end{align*}
	To show that this inequality is tight, note that equality holds for an all-ones matrix.
	Next, we prove the inequality 
	\[
	\min(m, d)^{\frac 1q - \frac 1p} \|\bM^T\|_{p, q}
	\leq \|\bM\|_{q, p},
	\] 
	for $q \leq p$.
	Applying Lemma~\ref{lemma:norm_ratio} twice gives
	\begin{equation}
	\label{eq:lin_app_d} \|\bM^T \|_{p,q}\leq \| \bM^T\|_{q,q}=\|\bM\|_{q,q}\leq d^{\frac 1q-\frac 1p}\|\bM\|_{p,q}.
	\end{equation}
	Again applying Lemma~\ref{lemma:norm_ratio} twice gives
	\begin{equation}
	\label{eq:lin_app_m} \|\bM^T\|_{p,q}\leq m^{\frac 1q-\frac 1p}\|\bM^T\|_{p,p}=m^{\frac 1q-\frac 1p}\|\bM\|_{p,p}\leq m^{\frac 1q-\frac 1p}\|\bM\|_{p,q}.
	\end{equation}
	(Lemma~\ref{lemma:norm_ratio} was presented in Section~\ref{sec:adv_linear_proof} and is proved in Appendix~\ref{sec:adv_linear}.) 
	Next, we show that \eqref{eq:lin_app_d} is tight if $d\leq m$ and that \eqref{eq:lin_app_m} is tight if $d \geq m$. If $d\leq m$, the bound is tight for the block matrix
	$\bM 
	= \left[
	\begin{smallmatrix} 
	\bI_{d \times d} \ | \ \mathbf 0
	\end{smallmatrix}
	\right]$,
	and, if 
	$d\geq m$, then the bound is tight for the block matrix
	$\bM 
	= \left[
	\begin{smallmatrix} \bI_{d\times d} \\[.075cm]
	\hline\\
	\mathbf 0
	\end{smallmatrix}
	\right].$
\end{proof}

\subsection{Constant Analysis}
\label{app:constants}

In this section, we study the constants in the two known bounds on the Rademacher complexity of linear classes for $1 < p \leq 2$. Specifically,
\begin{numcases}{\R_\cS(\cF_p)\leq}
\frac W m \sqrt{\du p-1} \|{\bX}\|_{\du p,2}
\label{eq:previous_linear_rc_2}\\[.25cm]
\frac{\sqrt{2}W}{m} \bigg[\frac{\Gamma( \tfrac{\du p + 1}{2} )}{\sqrt{\pi}} \bigg]^{\frac{1}{\du p}} \|{\bX^T}\|_{2,\du p}  \label{eq:new_linear_rc}
\end{numcases}
We will compare the constants in equations (\ref{eq:previous_linear_rc_2}) and (\ref{eq:new_linear_rc}), namely $\frac{\sqrt 2 W}m \big(\frac{\Gamma(\frac{\du p+1} 2 )}{\sqrt \pi}\big)^\frac 1 {\du p} $ and $\frac W m\sqrt{\du p-1}$. Since $\frac Wm$ divides both of these constants, we drop this factor and work with the expressions $c_1(p) \colon = \sqrt{\du p - 1}$ and $c_2(p)\colon = \sqrt 2 \big(\frac{\Gamma(\frac{\du p+1} 2 )}{\sqrt \pi}\big)^\frac 1 {\du p}$.
To start, we first establish upper and lower bound on $c_2(p)$.  
\begin{lemma}
\label{lemma:f2_bound}Let $c_2(p)=\sqrt 2 \big(\frac{\Gamma(\frac{\du p+1} 2 )}{\sqrt \pi}\big)^\frac 1 {\du p}$. Then the following inequalities hold:
\[
e^{-\frac 12}\sqrt{ \du p}\leq c_2(p)\leq  e^{-\frac 12} \sqrt {\du p+1}. 
\]
\end{lemma}
\begin{proof}
For convenience, we set $q = \du p$, $f_1(q)=c_1(p)$, $f_2(q)=c_2(p)$.
Next, we recall a useful inequality \citep{OlverLozierBoisvertClark2010} bounding the gamma function:
\begin{align}
1<(2\pi)^{-\frac 12}x^{\frac 12-x}e^x\Gamma(x)<e^{\frac 1 {12x}}.\label{eq:gamma_bound}
\end{align}

We start with the upper bound.
If we apply the right-hand side inequality of (\ref{eq:gamma_bound}) to $\Gamma(\frac {q+1}2)$ we get the following bound on $f_2(q)$:
\begin{equation}
\label{eq:f2_bound}
f_2(q)\leq 2^\frac 1 {2q} e^{-\frac 12} \sqrt{q+1} e^{-\frac 1{2q}+\frac 1 {6(q+1)q}}
\end{equation}
It is easy to verify that, 
\begin{equation}
\label{eq:upper_bound_factor}
2^\frac 1{2q}e^{-\frac 1 {2q} +\frac 1 {6q(q+1)}}=e^{\frac 1q(\frac {\ln 2-1}2+\frac 1 {6q(q+1)})}.
\end{equation}
Furthermore, the expression $(\frac {\ln 2-1}2+\frac 1 {6q(q+1)})$ decreases with increasing $q$. At $q=2$, it is negative, which implies that (\ref{eq:upper_bound_factor}) is less than 1 for $q\geq 2$. Hence
\begin{equation*}
f_2(q) \leq e^{-\frac 12} \sqrt{q+1} 
\end{equation*}

Next, we prove the lower bound.
Applying the lower bound of (\ref{eq:gamma_bound}) to $\Gamma(\frac{q+1}2)$ results in
\begin{equation*}
f_2(q)\geq e^{-\frac 12}\sqrt q \left(e^{-\frac 1{2q}(\log 2 -1)}\sqrt{1+\frac 1 q}\right).
\end{equation*} 
We will establish that $\Big(e^{-\frac 1{2q}(\log 2 -1)}\sqrt{1+\frac 1 q}\Big)\geq 1$, which will complete the proof of the lower bound. We prove this statement by showing that
\[
\left(e^{-\frac 1{2q}(\log 2 -1)}\sqrt{1+\frac 1 q}\right)^2
= e^{-\frac 1 q(\log 2 -1)} \left( 1+\frac 1q \right)\geq 1.
\]
By applying some elementary inequalities
\begin{align*}
e^{-\frac 1 q(\log 2 -1)}\left(1+\frac 1q\right)&\geq \left( \frac 1q(\log 2-1)+1\right) \left( 1 + \frac 1q \right) & (\text{using }e^x\geq 1+x)\\
&=1+\frac 1q \left( \log (2) -\frac{1-\log (2)}q\right)\\
&\geq 1
\end{align*}
The last inequality follows since $\Big( \log (2) -\frac{1-\log (2)}q\Big)$ increases with $q$, and is positive at $q = 2$. 
\end{proof}
Lastly, we establish our main claim that $c_2(p) \leq c_1(p)$.
\begin{lemma}
\label{lemma:constant-comparison}
Let $c_1(p)=\sqrt{\du p-1}$ and $c_2(p)=\sqrt 2 \big(\frac{\Gamma(\frac{\du p+1} 2 )}{\sqrt \pi}\big)^\frac 1 {\du p}$.
Then
\[
c_2(p) \leq c_1(p),
\]
for all $1 \leq p \leq 2$.
\end{lemma}
\begin{proof}
 For convenience, set $q=\du p$, $f_1(q)=c_1(p)$, and $f_2(q)=c_2(p)$. First note that $f_1(2)=f_2(2)$. Next, we claim $\frac d {dq} f_1(q)\geq \frac d {dq} f_2(q)$ for $q\geq 2$, and this implies that $c_2(p)\leq c_1(p)$ for $1\leq p\leq 2$. 

The rest of this proof is devoted to showing that $\frac d {dq} f_1(q)\geq \frac d {dq} f_2(q)$. Upon differentiating we get that $f_1'(q)=\frac 1 {2\sqrt{q-1}}$. Next, we will differentiate $f_2$. To start, we recall that the digamma function $\psi$ is defined as the logarithmic derivative of the gamma function, $\psi(x)=\frac d {dx} (\log\Gamma(x))=\frac{\Gamma'(x)}{\Gamma(x)}$.

Now we state a useful inequality~(see Equation~$2.2$ in \citet{Alzer1997})
 bounding the digamma function, $\psi(x)$. 
\begin{align}
    \psi(x)\leq \log(x)-\frac 1{2x}
\label{eq:digamma_bound}
\end{align}

Now we differentiate $\ln f_2$:
\begin{align*}
\frac d {dq} (\ln f_2(q)) &=\frac{ \frac q2 \psi(\frac {q+1} 2)-(\ln(\Gamma(\frac {q+1} 2))-\ln(\sqrt \pi))}{q^2}\\
&\leq\frac{ \frac q2(\log(\frac {q+1}2-\frac 1 {q+1})-(\ln(\Gamma(\frac{q+1}2))-\ln\sqrt \pi)}{q^2}&\,\text{(by (\ref{eq:digamma_bound}))}\\
&\leq \frac {\frac q2(\log \frac {q+1}2 -\frac 1 {q+1})-(\frac 12 \ln2 +\frac q2\log \frac {q+1}2-\frac {q+1}2)}{q^2}&\text{(by the left-hand equality in (\ref{eq:gamma_bound}))}\\
&= \frac 1 {2q}+\frac 1 {q^2}\Big(\frac 1 {2(q+1)}-\frac 12\log2\Big)\\
&\leq \frac 1 {2q}.
\end{align*}
The last line follows since we only consider $q\geq 2$ and $\frac 1 {2(q+1)}-\frac 12\ln 2\leq 0$ in this range.
Finally, the fact that $\frac d {dq}(\ln f_2(q))=f_2'(q)/f_2(q)$ implies 
\begin{align*}
f_2'(q) &= f_2(q) \frac d {dq} (\ln f_2(q))\\
&\leq \frac 1 {2q} f_2(q)&\,\text{(by }\frac d {dq}(\ln f_2(q))\leq \frac 1 {2q} \text{)}\\
&\leq \frac{e^{-\frac 12} \sqrt{q+1}}{2q}&\,\text{(by applying the upper bound in Lemma~\ref{lemma:f2_bound})}\\
&=\frac 1 {2\sqrt{q-1}}\frac{e^{-\frac 12}\sqrt{(q+1)(q-1)}}q\\
&\leq e^{-\frac 12} \frac 1{2\sqrt{q-1}}&\,\text{(using }q^2-1\leq q^2)\\
&\leq \frac 1 {2\sqrt{q-1}}=f_1'(q)&\,(\text{using }e^{-\frac 12}<1).
\end{align*} 
\end{proof}
\section{Proof of Theorem~\ref{th:adversarial_linear}}
\label{sec:adv_linear} 
\label{app:big}
In this section, we give a detailed proof of Theorem~\ref{th:adversarial_linear}.
We start with the following lemma that characterizes the nature of adversarial perturbations.

\begin{lemma}
\label{lemma:perturb_increasing}
Let $g$ be a nondecreasing function, $\bx,\bw\in \mathbb R^d$, and $y\in \{\pm 1\}$. Then 
\[\inf_{\|\bx-\bx'\|_r\leq \e} yg( \bw\cdot \bx)= yg(\bw\cdot \bx-\e y \|\bw\|_\du r)\] 
\end{lemma}
\begin{proof}
First note that
\[\inf_{\|\bx-\bx'\|_r\leq \e} yg( \bw\cdot \bx)=\inf_{\|\bs\|_r\leq 1} yg( \bw\cdot \bx+\e \bw \cdot \bs)\]
If $y=1$,
\begin{align}
\inf_{\|\bs\|_r\leq 1} g( \bw\cdot \bx+\e \bw \cdot \bs) &= g( \bw\cdot \bx+\inf_{\|\bs\|_r\leq 1}\e \bw \cdot \bs)&(g\text{ is nondecreasing})\nonumber\\ 
&=g(\bw\cdot \bx-\e \|\bw\|_\du r )&\text{(definition of dual norm)}\nonumber\\
&=yg( \bw\cdot \bx-\e y \|\bw\|_\du r )&(y=1)\nonumber
\end{align}
Similarly, if $y=-1$,
\begin{align}
\inf_{\|\bs\|_r\leq 1} -g( \bw\cdot \bx+\e \bw \cdot \bs) &= -g( \bw\cdot \bx+\sup_{\|\bs\|_r\leq 1}\e \bw \cdot \bs)&(-g\text{ is non-increasing})\nonumber\\ 
&=-g(\bw\cdot \bx+\e \|\bw\|_\du r )&\text{(definition of dual norm)}\nonumber\\
&=yg( \bw\cdot \bx-\e y \|\bw\|_\du r )&(y=-1)\nonumber
\end{align}
\end{proof}
Before proceeding to the proof of Theorem~\ref{th:adversarial_linear}, we formally establish Lemma~\ref{lemma:norm_ratio} and Lemma~\ref{lemma:big} from Section~\ref{sec:lin_rc}.
\begin{proof}[Proof of Lemma~\ref{lemma:norm_ratio}]
We prove that 
if $p\geq \du r$, then 
\[
\sup_{\| \bw\|_p\leq 1}\| \bw\|_{\du r}=d^{1-\frac 1r -\frac 1p}
\]
and otherwise,
\[
\sup_{\| \bw\|_p\leq 1}\| \bw\|_{\du r}= 1.
\]
If $p\geq \du r$, by H{\"o}lder's generalized inequality with $\frac 1 {r^*} = \frac 1p + \frac 1s$,
\[
\sup_{\| \bw\|_p\leq 1} \| \bw\|_\du r \leq \sup_{\| \bw\| _p \leq 1} \| \one\|_s \| \bw\|_p = \| \one\|_s = d^\frac 1s = d^{\frac 1 {r^*} - \frac 1 p} = d^{1 - \frac 1r - \frac 1p}.
\]
Equality holds at the vector $\frac 1 {d^\frac 1p}\one$, and this implies that the inequality in the line above is an equality. Now for $p\leq \du r$, $\| \bw\|_p\geq \| \bw\|_\du r$, implying that $\sup_{\| \bw\|_p\leq 1}\| \bw\|_{\du r}\leq  1$. Here, equality is achieved at a unit vector $\be_1$. 
\end{proof}
\begin{proof}[Proof of Lemma~\ref{lemma:big}]
Recall that $v_\bsigma = \frac 1 m \sum_{i=1}^m \sigma_i$. Then, 
in view of the symmetry $v_{-\bsigma} = -v_{\bsigma}$, we can write
\[
\Ex[\bsigma]{ \sup_{\| \bw\|_p\leq W} \e v_\bsigma \| \bw\|_{\du r}}
= \e W\E_\bsigma\left[ {\sup_{\| \bw\|_p\leq 1} v_\bsigma \| \bw\|_\du r}\right]
= \frac {\e W}{2} \Ex[\bsigma]{ \sup_{\| \bw\|_p\leq 1} |v_\bsigma| \| \bw\| _\du r}.
\]
By Lemma~\ref{lemma:norm_ratio}, we have
\begin{equation}
\label{eq:prev}
\frac 12 \Ex[\bsigma] {\sup_{\| \bw\|_p\leq 1} |v_\bsigma| \| \bw\|_{\du r}} =\frac 12 \max(d^{1-\frac 1p-\frac 1r},1) \Ex[\bsigma]{ |v_\bsigma|}.
\end{equation}
Now, by Jensen's inequality and $\E[\sigma_i \sigma_j] = \E[\sigma_i] \E[\sigma_j] = 0$ for $i \neq j$, we have
\[
\E_{\bsigma}[|\bv_\bsigma|]
= \Ex[\bsigma]{\left| \sum_{i=1}^m \sigma_i \right|} 
\leq \sqrt{\Ex[\bsigma]{\Big(\sum_{i=1}^m \sigma_i \Big)^2 }}
= \sqrt{\Ex[\bsigma]{ m + \sum_{i \neq j} \sigma_i\sigma_j}}
=\sqrt m.
\] 
Furthermore, by Khintchine's inequality \citep{Haagerup1981}, the following lower bound holds:
\[
\Ex[\bsigma]{\left| \sum_{i=1}^m \sigma_i\right|} \geq \sqrt{\frac{m}{2}}.
\]
Substituting these upper and the lower bounds into \eqref{eq:prev} completes the proof.
\end{proof}
We now proceed to prove Theorem~\ref{th:adversarial_linear}. Recall
from Section~\ref{sec:adv_linear_proof} that we seek to analyze
\begin{align}
\R_\cS(\wt{ \mathcal F}_p)
& = \E_{\bsigma}{\left[ \sup_{\| \bw\|_p\leq W} \frac 1m \sum_{i = 1}^m
  \sigma_i\inf_{\|{\bx_i-\bx_i'}\|_r\leq \e} y_i\la \bw,\bx_i'\ra
  \right]} \nonumber\\
& = \E_\bsigma { \left[ \sup_{\| \bw\|_p\leq W} \frac 1m \sum_{i=1}^m
  \sigma_i( y_i\la \bw,\bx_i\ra-\e \|\bw\|_{\du r})\right]} & \text{[by Lemma~\ref{lemma:perturb_increasing}]} \nonumber\\
& = \E_\bsigma{ \left[ \sup_{\| \bw\|_p\leq W}  \la \bw ,\bu_\bsigma\ra -\e v_\bsigma \| \bw\| _{\du r}\right]},
\label{eq:1-1}
\end{align} 
where we used the shorthand
$\bu_\bsigma = \frac 1m \sum_{i=1}^m y_i\sigma_i \bx_i$ and
$v_\bsigma = \frac 1m \sum_{i=1}^m \sigma_i$. The next two theorems
give upper and lower bounds on $\R_\cS(\wt{ \mathcal F}_p)$,
thereby proving Theorem~\ref{th:adversarial_linear}.

\begin{theorem}
\label{th:upper}
Let $\cF_p=\{\bx\mapsto \la \bx,\bw\ra\colon \|\bw\|_p\leq W\}$ and
$\wt \cF_p=\{\inf_{\|\bx'-\bx\|_r\leq \eps} f(\bx')\colon f\in
\cF_p\}$. Then, the following upper bound holds:
\[
\R_\cS (\wt \cF_p)\leq \R_\cS(\mathcal {F}_p)
+ \eps \frac{W}{2\sqrt m} d^{1-\frac 1r-\frac 1p}
\]

\end{theorem}
\begin{proof}
Using (\ref{eq:1-1}) and the sub-additivity of supremum we can write:
\begin{align*}
\R_\cS(\wt{ \mathcal F}_p) &= \Ex[\sigma ]{\sup_{\| \bw\|_p\leq W} \la \bw,\bu_\sigma\ra-\eps v_\bsigma \| \bw\|_{r^*} }\\
&\leq \Ex[\bsigma]{ \sup_{\| \bw\|_p\leq W} \la \bw,\bu_\bsigma\ra}+\Ex[\bsigma]{ \sup_{\| \bw\|_p\leq W}-\eps v_\bsigma\| \bw\|_\du r}\\
&=\R_\cS(\mathcal F_p)+\Ex[\bsigma]{\sup_{\| \bw\|_p\leq W} \eps v_\bsigma \| \bw\|_{\du r}}\\
&= \R_\cS(\mathcal {F}_p)+\frac 12 \eps \frac{W}{\sqrt m} d^{1-\frac 1r-\frac 1p} & \text{[by Lemma~\ref{lemma:big}]},
\end{align*}
which completes the proof.
\end{proof}

\begin{theorem} 
Let $\cF_p=\{\bx\mapsto \la \bx,\bw\ra\colon \|\bw\|_p\leq W\}$ and
$\wt \cF_p=\{\inf_{\|\bx'-\bx\|_r\leq \eps} f(\bx')\colon f\in
\cF_p\}$. Then, the following lower bound holds:
\[
\R_\cS(\wt{ \mathcal F}_p)\geq  \max\left(\R_\cS(\mathcal F_p),W
  \frac{\eps d^{1-\frac 1r-\frac 1p} }{ 2\sqrt{2 m}}\right)
\]
\end{theorem}
\begin{proof} The proof involves two symmetrization arguments. Since
  $-\bsigma$ follows the same distribution as $\bsigma$, we have the equality 
\begin{align}
\R_\cS(\wt \cF_p) 
& = \Ex[\bsigma]{  \sup_{\| \bw\|_p\leq W}\la \bw ,\bu_{-\sigma}\ra
  -\eps v_{-\sigma} \| \bw\| _{\du r}} 
= \Ex[\bsigma]{  \sup_{\| \bw\|_p\leq W}  -\la \bw ,\bu_\sigma\ra
  +\eps v_\bsigma \| \bw\| _{\du r}}. 
\label{eq:2}
\end{align}  
Similarly, $\bw$ can be replaced with $-\bw$, thus we have
\begin{equation}
\label{eq:3}
\R_\cS(\wt \cF_p)
= \Ex[\bsigma] {\sup_{\| \bw\|_p\leq W}  \la \bw ,\bu_\bsigma\ra +\eps
  v_\bsigma \| \bw\| _{\du r}}.
\end{equation} 
Averaging (\ref{eq:1-1}) and (\ref{eq:3}) and using the sub-additivity of the supremum gives 
\begin{align*}
\R_\cS(\wt \cF_p) &=\frac 12\Ex[\bsigma]{  \sup_{\| \bw\|_p\leq W}\la \bw ,\bu_\bsigma\ra -\eps v_\bsigma \| \bw \|_{\du r}}+\frac 12\Ex[\bsigma]{  \sup_{\| \bw\|_p\leq W}  \la \bw ,\bu_\bsigma\ra +\eps v_\bsigma \| \bw\| _{\du r}}
\geq \Ex[\bsigma]{ \sup_{\| \bw\|_p\leq W} \la \bw,\bu_\bsigma\ra}= W\R_\cS({\mathcal F_p}).
\end{align*}
Now, averaging (\ref{eq:2}) and (\ref{eq:3}), and using the
sub-additivity of supremum give:
\begin{align*}
\R_\cS(\wt \cF_p) 
& = \frac 12\Ex[\bsigma]{  \sup_{\| \bw\|_p\leq W}-\la \bw
  ,\bu_\bsigma\ra +\eps v_\bsigma \| \bw \|_{\du r}} + \frac 12\Ex[\bsigma]{  \sup_{\| \bw\|_p\leq W}  \la \bw ,\bu_\bsigma\ra +\eps v_\bsigma \| \bw \|_{\du r}}\\
& \geq \Ex[\bsigma]{ \sup_{\| \bw\|_p\leq W} v_\bsigma \| \bw\|_\du r}
\geq \frac 1 {2\sqrt{2 m}} \eps d^{1-\frac 1p -\frac 1r}, & \text{[from Lemma~\ref{lemma:big}].}
\end{align*}
which completes the proof.
\end{proof}


\section{Adversarial Rademacher Complexity of ReLU}
\label{app:relu_proofs}

In this section, we prove upper and lower bounds on the
Rademacher complexity of the ReLU unit. We will use the notation
$z_+ = \max(z, 0)$, for any $z \in \Rset$. 
We use the family of functions ${\mathcal G}_p$ defined in \eqref{eq:relu_function_class}
with the corresponding adversarial class $\wt{ \mathcal G}_p$:
\[
\wt{ \mathcal G}_p
=\set[\big]{ (\bx, y) \mapsto \inf_{\| \bs\|_r\leq \eps}
y(\bw\cdot(\bx+\bs))_+ \colon \| \bw \|_p \leq W, y \in \{-1, +1\} }.
\]
Since $z \mapsto z_+$ is non-decreasing, by
Lemma~\ref{lemma:perturb_increasing}, $\wt{ \mathcal G}_p$ can be
equivalently expressed as follows:
\[
\wt \cG_p=\{(\bx,y)\mapsto y(\bw\cdot \bx-\eps y\| \bw\|_\du
r)_+\colon \| \bw\|_p\leq W, y\in\{-1,1\}\}.
\]
In view of that, the adversarial Rademacher complexity of the ReLU
unit can be written as follows:
\begin{align} 
\wt \R_\cS(\cG_p) = \R_\cS(\wt \cG_p)
& = \E_\bsig\left[{\sup_{\| {\bw}\|_p\leq W} \frac 1m \sum_{i=1}^m
  {\sigma_i y_i ( \bw\cdot \bx_i-y_i\eps \|\bw\|_\du r)_+}}\right] 
= \E_\bsig\left[{\sup_{\| {\bw}\|_p\leq W} \frac 1m \sum_{i=1}^m
  \sigma_i { ( \bw\cdot \bx_i-y_i\eps \| \bw\|_\du r)_+}}\right]. 
\label{eq:4}
\end{align}

\subsection{Upper Bounds}

\begin{reptheorem}{th:data_dep_reLUupper}
Let $\cG_p$ the class defined in
\eqref{eq:relu_function_class} and let $\cF_p$ be the
linear class as defined in \eqref{eq:linear_function_class}. Then,
given a sample $\cS= \{(\bx_1, y_1), \dots, (\bx_m, y_m)\}$, the
adversarial Rademacher complexity of $\cG_p$ can be bounded as
follows:
\begin{align*}
\wt \R_\cS(\cG_p)
& \leq {\mathfrak R }_{T_\eps}(\mathcal F_p)+\eps \frac{W}{2\sqrt{m}} \max(1,d^{1-\frac 1r-\frac 1p}),
\end{align*}
where $T_\eps=\{i\colon y_i=-1\text{ or } (y_i=1\text{ and }\|{\bx_i}\|_r> \eps)\}$.
\end{reptheorem}

\begin{proof}
  Consider an index $i \in [m]$ such that $i \not \in T_\eps$, so that
  $\| {\bx_i}\|_r\leq \eps$ and $y_i=1$. Then, by H{\" o}lder's
  inequality, we have
\[
y_i \bw\cdot \bx_i-y_i\eps \| \bw\|_{\du r}
= \| \bw \|_{\du r} \left( \frac \bw{\| \bw\|_\du r}\cdot \bx_i-\eps\right)
\leq \| \bw\|_\du r(\| {\bx_i}\|_r-\eps)\leq 0,
\] 
and therefore $( \bw\cdot \bx_i-\eps \| \bw\|_\du r)_+ = 0$ for all
$\bw$ with $\| \bw\|_p \leq W$. 
Thus, using the expression \eqref{eq:4}, we can write:
\begin{align*}
\R_\cS(\wt\cG_p) 
& = \Ex[\bsigma]{\sup_{\| \bw\|_p\leq W}\frac 1m \sum_{i\in T_\eps}\sigma_i(y_i \bw\cdot \bx_i-\eps \| \bw\|_\du r)_+}\\
& \leq \Ex[\bsigma]{\sup_{\| \bw\|_p\leq W}\frac 1m \sum_{i\in
  T_\eps}\sigma_i(y_i \bw\cdot \bx_i-\eps \| \bw\|_\du r)} & 
\text{($1$-Lipschitzness of $z \mapsto z_+$)}\\
& = \frac{|T_\eps|}{m}\R_{T_\eps}(\wt {\mathcal F}_p)\\
& \leq {\mathfrak R }_{T_\eps}(\mathcal F_p)+\eps \frac{W}{2\sqrt{m}}
  \max(1,d^{1-\frac 1r-\frac 1p}),
& \text{(Theorem~\ref{th:adversarial_linear})}
\end{align*}
which completes the proof.
\end{proof}

\subsection{Lower Bounds}

\begin{reptheorem}{th:data_dep_ReLUlower}
  Let $\cG_p$ be the class as defined in
  \eqref{eq:relu_function_class}. Then it holds that
\[
\wt\R_\cS(\cG_p) \geq\frac W{2\sqrt 2 m} \sup_{\| \bs\|_p=1} \bigg(\sum_{i\in T_{\eps,\bs}}(\la \bs,\bx_i\ra-\eps y_i \| \bs\|_\du r)^2\bigg)^\frac 12
\]
where $T_{\eps,\bs} = \set{i\colon \la \bs,\bx_i\ra-y_i\eps \| \bs\|_\du r>0 }$.
\end{reptheorem}

\begin{proof}
By definition of the supremum, we can write:
\[
\R_\cS(\cG_p)
=\Ex[\bsig]{\sup_{\| {\bw}\|_p\leq W} \frac 1m
  \sum_{i=1}^m \sigma_iy_i(\la \bw,\bx_i\ra -y_i\eps\| \bw\|_{\du r})_+}
= \Ex[\bsigma]{\sup_{\substack{B \leq W \\ \| \bs \|_p = 1}} \frac Bm
  \sum_{i=1}^m \sigma_i(\la \bs,\bx_i\ra -\eps y_i\| \bs\|_\du r)_+}.
\]
Now, for a fixed $\bs$, it is straightforward to take the supremum
over $B$: if the quantity
$\sum_{i=1}^m \sigma_i(\la \bs,\bz_i\ra -\eps \| \bs\|_\du r)_+$ is
positive, the expression is maximized by taking $B = W$; otherwise it is
maximized by $B = 0$. Thus, we have
\begin{align*}
\Ex[\bsigma]{\sup_{\substack{B<W\\\| \bs\|_p=1}} \frac Bm \sum_{i=1}^m
  \sigma_i(\la \bs,\bx_i\ra -y_i\eps \| \bs\|_\du r)_+} 
& =\frac Wm \Ex[\bsigma]{\sup_{\| \bs\|_p=1} \max\left(0,\sum_{i=1}^m \sigma_i(\la \bs,\bx_i\ra-\eps \| \bs\|_\du r)_+\right)}\\
& \geq \frac W m\sup_{\| \bs\|_p= 1}\Ex[\bsigma]{ \max\left(0,\sum_{i=1}^m \sigma_i (\la \bs,\bx_i\ra -\eps y_i\| \bs\|_\du r)_+\right)}\\
& =\frac W{2m} \sup_{\| \bs\|_p= 1} 
\Ex[\bsigma]{\left|\sum_{i=1}^m \sigma_i(\la \bs,\bx_i\ra-\eps y_i \| \bs\|_\du r)_+\right|}\\
& = \frac W{2m} \sup_{\| \bs\|_p= 1} \Ex[\bsigma]{\left|\sum_{i\in T_{\eps,\bs}}\sigma_i(\la \bs,\bx_i\ra-y_i\eps \| \bs\|_\du r)\right|}.
\end{align*}
Next, by the Khintchine-Kahane inequality \citep{Haagerup1981}, the
following lower bound holds:
\begin{align*}
\frac W{2m} \sup_{\| \bs\|_p= 1} \Ex[\bsigma]{\left|\sum_{i\in T_{\eps, \bs}} \sigma_i(\la \bs,\bx_i\ra-\eps y_i \| \bs\|_\du r)\right|} &\geq \frac W{2\sqrt 2 m} \sup_{\| \bs\|_p= 1}\left( \Ex[\bsigma]{\left(\sum_{i\in T_{\eps, \bs}} \sigma_i(\la \bs,\bx_i\ra-y_i\eps \| \bs\|_\du r)\right)^2}\right)^\frac 12 \\
& = \frac W{2\sqrt 2 m} \sup_{\| \bs\|_p=1}\left( \Ex[\bsigma]
  {\sum_{i, j \in T_{\eps, \bs}}  \sigma_i\sigma_j(\la \bs,\bx_i\ra-\eps y_i \| \bs\|_\du r)(\la \bs,\bx_j\ra -\eps y_i \| \bs\|_\du r)}\right)^\frac 12\\
& = \frac W{2\sqrt 2 m} \sup_{\| \bs\|_p=1}\left({ \sum_{i\in T_{\eps, \bs}}(\la \bs,\bx_i\ra-y_i\eps \| \bs\|_\du r)^2 }\right)^\frac 12,
\end{align*}
which completes the proof.
\end{proof}

\section{Adversarial Rademacher for Neural Nets with One Hidden Layer
  with a Lipschitz Activation Function}
\label{app:nn-lip}

In this section, we present an upper bound on the adversarial Rademacher complexity of one-layer neural networks with an activation function satisfying some reasonable requirements. Our analysis uses the notion of coverings.

\begin{definition}[$\epsilon$-covering]
Let $\e>0$ and let$(V, \|\cdot\|)$ be a normed space.
$\cC \subseteq V$ is an \emph{$\epsilon$-covering of $V$} if for any $v \in
V$, there exists $v' \in \cC$ such that $\|v - v'\| \leq \epsilon$.
\end{definition}
In particular, we will use the following lemma regarding the
size of coverings of balls of a certain radius in a normed space.

\begin{lemma}
\label{lemma:covering}\cite{MohriRostamizadehTalwalkar2018}
Fix an arbitrary norm $\|\cdot\|$ and let $\mathcal B$ be the ball radius $R$ in this norm. Let $\cC$ be a smallest possible $\e$-covering of $\mathcal B$. Then
\[
|\cC|\leq \left(\frac {3R} \e \right)^d
\]
\end{lemma}
Next, we give the proof of the main theorem of this section.

\begin{reptheorem}{th:one_layer_nn_rc}
Let $\rho$ be a function with Lipschitz constant $L_\rho$ satisfying
$\rho(0) = 0$ and consider perturbations in $r$-norm.  Then, the
following upper bound holds for the adversarial Rademacher complexity
of $\cG^n_p$:
\begin{align*}
\wt \R_\cS(\cG_p^n)
& \leq L_\rho \bigg[ \frac{W\Lambda \max(1,d^{1-\frac 1p-\frac 1r})(\|\bX\|_{r,\infty}+\e)}{\sqrt m} \bigg]
 \left( 1 + \sqrt{d(n+1)\log(9 m)}\right).
\end{align*}
\end{reptheorem}

\begin{proof}
  Let $\cC_1$ be a covering of the $\ell_1$ ball of radius $\Lambda$
  with $\ell_1$ balls of radius $\delta_1$ and $\cC_2$ a covering of
  the $\ell_p$ ball of radius $W$ with $\ell_p$ balls of radius
  $\delta_2$. We will later choose $\delta_1$ and $\delta_2$ as
  functions of $m$, $W$, and $\Lambda$. For any $\bx$, 
define $\wt f(\bx)$ and $\wt f^c(\bx)$ as follows:
\[
\wt f(\bx) = \inf_{\norm{\bx'-\bx}_r\leq \e}y\sum_{j=1}^n u_j
\rho(\bw_j\cdot \bx') 
\quad \text{and} \quad
\wt f^c(\bx) = \inf_{\norm{\bx'-\bx}_r\leq \e}y\sum_{j=1}^n u^c_j
\rho(\bw^c_j\cdot \bx'),
\]
where $\bu^c$ is the closest element to $\bu$ in $\cC_1$ and $\bw^c$
is the closest element to $\bw$ in $\cC_2$.  Define $\e'$ as follows:
\[
\e' = \sup_{i \in [m]} \sup_{\substack{\|\bu\|_1\leq
    \Lambda\\ \| \bw \|_p \leq W}} |\wt f(\bx_i) - \wt f^c(\bx_i)|.
\]
One can bound the Rademacher complexity of the whole class $\cG_p^n$ in terms of the Rademacher complexity of this same class restricted to $\bu \in \cC_1$ and $\bw_j\in \cC_2$.
\begin{align}
\wt \R_\cS(\cG_p^n)
& = \E_{\sigma}\bigg[{\sup_{\substack{\| \bu\|_1\leq \Lambda\\ \|
  \bw\|_j\leq W}}\frac 1m \sum_{i=1}^m
  \sigma_i\inf_{\|\bx_i-\bx_i'\|_r\leq \eps} y_i\sum_{j=1}^n
  u_j\rho(\bw_j\cdot \bx_i')}\bigg] \nonumber\\
& \leq \E_{\sigma}\bigg[{\sup_{\substack{\norm \bu^c\in \cC_1\\
  \bw^c_j\in \cC_2}}\frac 1m \sum_{i=1}^m
  \sigma_i\inf_{\|\bx_i-\bx_i'\|_r\leq \eps} y_i\sum_{j=1}^n
  u^c_j\rho(\bw^c_j\cdot \bx_i')}\bigg]+\e' 
\label{eq:covering_rc}
\end{align} 
Then, by Massart's lemma, the first term in \eqref{eq:covering_rc} can
be bounded as follows:
\begin{equation}
\label{eq:massart_application}
\E_{\sigma}\bigg[{\sup_{\substack{\norm \bu^c\in \cC_1\\ \bw^c_j\in \cC_2}}\frac 1m \sum_{i=1}^m \sigma_i\inf_{\|\bx_i-\bx_i'\|_r\leq \eps} y_i\sum_{j=1}^n u^c_j\rho(\bw^c_j\cdot \bx_i')}\bigg]\leq \frac {K \sqrt {2\log (|\cC_1| |\cC_2|^n)}}{ m}\end{equation}
with 
\[
K^2 =\sup_{\substack{\bw^c_j\in \cC_2\\ \bu^c \in \cC_1}}
  \sum_{i=1}^m \left(\inf_{\|\bx_i-\bx_i'\|_r\leq \eps} y_i
    \sum_{j=1}^n u_j^c\rho(\bw_j^c\cdot \bx_i')\right)^2. 
\]
We will show the following upper bound for $K$:
\begin{equation}
\label{eq:K-bound}
K\leq \sqrt m \Lambda W\max \Big( 1, d^{1-\frac 1r-\frac 1p} (
\|\bX\|_{r,\infty} + \e) \Big).
\end{equation}
Let $\bx_*^c$ be the minimizer of $f^c(\bx)$ within an $\epsilon$-ball around $\bx$. Since $\wt f^c$ is continuous and the closed unit $r$-ball is compact,
the extreme value theorem implies that $\bx_*^c$ exists. Then 
\begin{equation}
\label{eq:def-x-c-star}
\wt f^c(\bx) = y\sum_{j=1}^n u^c_j \rho(\bw^c_j\cdot \bx_*^c)
\end{equation}

We then apply the following inequalities:
\begin{align}
\left| y_i \sum_{j=1}^n u_j^c\rho(\bw_j^c\cdot \bx_{i*}^c ) \right|&\leq \sum_{j=1}^n |u_j^c| |\rho(\bw_j^c\cdot \bx_{i*}^c)|&\text{(triangle inequality)}\nonumber\\
&= \sum_{j=1}^n |u_j^c| |\rho(\bw_j^c\cdot \bx_{i*}^c)-\rho(0)|&(\rho(0)=0\text{ assumption)}\nonumber\\
&\leq L_\rho\sum_{j=1}^n |u_j^c| |\bw_j^c\cdot \bx_{i*}^c|&\text{(Lipschitz property)}\nonumber\\ 
&\leq L_\rho\sum_{j=1}^n |u_j^c||\bw_j^c\|_p\|\bx_{i*}^c\|_\du p&\text{(H{\"o}lder's inequality)}\nonumber\\
&\leq L_\rho\sum_{j=1}^n |u_j^c| W\|\bx_{i*}^c\|_\du p&(\|\bw_j\|\leq W)\nonumber\\
&\leq L_\rho\Lambda W\|\bx_{i*}^c\|_\du p&(\|\bu\|\leq \Lambda)\nonumber\\
&\leq L_\rho\Lambda W (\|\bX\|_{r,\infty}+\e)\max(1,d^{1-\frac 1r
  -\frac 1p}).
\label{eq:nn-massart-2}
\end{align}
The last inequality is justified by the following, where we use 
the triangle inequality and Lemma~\ref{lemma:norm_ratio}:
\begin{align}
\|\bx_{i*}^c\|_p 
& \leq \max(1,d^{1-\frac 1p-\frac 1r}) \|\bx_{i*}^c\|_r \nonumber \\
& \leq \max(1,d^{1-\frac 1p-\frac 1r})(\|\bx_i\|_r+\| \bx_{i*}^c-\bx_i\|_r)\nonumber \\
& \leq\max(1,d^{1-\frac 1r-\frac 1p}) ( \max_{i\in [m]} \|\bx_i\|_r+\e) \nonumber \\
& \leq \max(1,d^{1-\frac 1r-\frac 1p}) (  \|\bX\|_{r,\infty}+\e).
\label{eq:xci-star-bound}
\end{align}
Equation (\ref{eq:nn-massart-2}) implies the desired bound \eqref{eq:K-bound} on $K$. Next, plugging in the bound from
Lemma~\ref{lemma:covering} in \eqref{eq:massart_application}, we obtain
\begin{align}
\label{eq:bound-on-massart}
\wt \R_\cS(\cG_p^n)
& \leq \frac {L_\rho\Lambda W \max(1, d^{1-\frac
1p-\frac 1r}) (\|\bX\|_{r,\infty}+\e)}{\sqrt m}\sqrt{2d\log \left(\frac {3\Lambda}{\delta_1}
\right) + 2nd \log \left( \frac{3W}{\delta_2}\right)} + \epsilon'.
\end{align}
We now turn our attention to estimating $\e'$.  Similar to 
\eqref{eq:def-x-c-star}, we define $\bx_*$ as the minimizer of
$\wt f(\bx)$ within an $\epsilon$-ball around $\bx$ where
\[
\wt f(\bx) = y\sum_{j=1}^n u_j \rho(\bw_j\cdot \bx_*).
\]
We decompose the difference between $\wt f(\bx_i)$ and
$\wt f^c(\bx_i)$ and bound each piece separately:
\begin{align}
\label{eq:bound_eps}
\wt f(\bx_i) - \wt f^c(\bx_i) 
& = \left( y\sum_{j=1}^n u_j\rho(\bw_j\cdot \bx_{i*}) - y\sum_{j=1}^n u_j\rho(\bw_j^c\cdot \bx_{i*}^c)\right)
+ \left(y\sum_{j=1}^n u_j\rho(\bw_j^c\cdot \bx_{i*}^c) - y\sum_{j=1}^n u^c_j\rho(\bw_j^c\cdot \bx_{i*}^c)\right).
\end{align}
The first term above can be bounded as follows:
\begin{align}
& y\sum_{j=1}^n u_j\rho(\bw_j\cdot \bx_{i*}) - y\sum_{j=1}^n
  u_j\rho(\bw_j^c\cdot \bx_{i*}^c) \\
& \leq y\sum_{i=1}^n u_j\rho(\bw_j\cdot \bx^c_{i*}) - y\sum_{j=1}^n u_j\rho(\bw_j^c\cdot \bx_{i*}^c)&\text{(infimum of first sum at }\bx_*)\nonumber\\
& \leq \sum_{j=1}^n |u_j| |\rho(\bw_j\cdot
  \bx_{i*}^c)-\rho(\bw_j^c\cdot \bx_{i*}^c)|&\text{(triangle
                                              inequality)} \nonumber\\
& \leq L_\rho\sum_{j=1}^n |u_j| |(\bw_j-\bw_j^c)\cdot
  \bx_{i*}^c|&\text{(Lipschitz property)} \nonumber\\
& \leq L_\rho\sum_{j=1}^n |u_j| \|\bw_j-\bw_j^c\|_p\|\bx_{i*}^c\|_\du
  p&\text{(H{\"o}lder's inequality)} \nonumber\\
& \leq L_\rho\sum_{j=1}^n |u_j| \delta_2\|\bx_{i*}^c\|_\du
  p&(\|\bw_j-\bw_j^c\|\leq \delta_2) \nonumber\\
& \leq L_\rho\sum_{j=1}^n |u_j| \delta_2\max(1,d^{1-\frac 1p-\frac
  1r})(\|\bX\|_{r,\infty}+\eps)&\text{(equation
                                 (\ref{eq:xci-star-bound}))} \nonumber\\
& \leq L_\rho\Lambda \delta_2
  (\|\bX\|_{r,\infty}+\eps)\max(1,d^{1-\frac 1p-\frac
  1r}). & (\|\bu\|_1\leq \Lambda)
\label{eq:first_term_bound}
\end{align}
Similarly we can bound the second term in \eqref{eq:bound_eps} as follows:
\begin{align}
& y\sum_{j=1}^n u_j\rho(\bw_j^c\cdot \bx_{i*}^c)-y\sum_{j=1}^n
  u^c_j\rho(\bw_j^c\cdot \bx_{i*}^c) \\
& \leq \sum_{j=1}^n |u_j-u_j^c| |\rho(\bw_j\cdot \bx_{i*}^c)|&\text{(triangle inequality)}\nonumber\\
& = \sum_{j=1}^n |u_j-u_j^c| |\rho(\bw_j\cdot \bx_{i*}^c)-\rho(0)|&(\rho(0)=0\text{ assumption)}\nonumber\\
& \leq L_\rho \sum_{j=1}^n |u_j-u_j^c| |\bw_j\cdot \bx_{i*}^c|&\text{(Lipschitz property)}\nonumber\\
& \leq L_\rho \sum_{j=1}^n |u_j-u_j^c| \|\bw_j\|_p\|\bx_{i*}^c\|_\du p&\text{(H{\"o}lder's inequality)}\nonumber\\
& \leq L_\rho \sum_{j=1}^n |u_j-u_j^c| W\|\bx_{i*}^c\|_\du p&(\|\bw_j\|\leq W)\nonumber\\
& \leq L_\rho \sum_{j=1}^n |u_j-u_j^c| W (\|\bX\|_{r,\infty}+\eps)\max(1,d^{1-\frac 1p-\frac 1r})&\text{(equation (\ref{eq:xci-star-bound}))}\nonumber\\
& \leq L_\rho \delta_1W (\|\bX\|_{r,\infty}+\eps)\max(1,d^{1-\frac 1p-\frac 1r})&(\|\bu-\bu^c\|_1\leq \delta_1)\label{eq:second_term_bound}
\end{align}
Combining equations (\ref{eq:first_term_bound}) and
(\ref{eq:second_term_bound}) results in
\[
\wt f(\bx_i)-\wt f^c(\bx_i)\leq L_\rho
  (\|\bX\|_{r,\infty}+\eps)\max(1,d^{1-\frac 1p-\frac
    1r})(W\delta_1+\Lambda\delta_2).
\]
By a similar analysis, one can also show that
$\wt f^c(\bx_i)-\wt f(\bx_i)
\leq L_\rho(\|\bX\|_{r,\infty}+\eps)\max(1,d^{1-\frac 1p -\frac
  1r})(W\delta_1+\Lambda \delta_2)$. Therefore
\begin{equation}
\label{eq:eps_prime_bound}
\eps'\leq L_\rho
  (\|\bX\|_{r,\infty}+\eps) \max(1,d^{1-\frac 1p-\frac
    1r})(W\delta_1 + \Lambda\delta_2)
\end{equation}

Combining equations (\ref{eq:eps_prime_bound}) and
(\ref{eq:bound-on-massart}) and choosing
$\delta_1=\frac \Lambda{2\sqrt m}$ and $\delta_2=\frac W{2\sqrt m}$ yield
\begin{align*}
\wt \R_\cS(\cG_p^n) \leq 
\left( \frac{L_\rho W \Lambda \max(1,d^{1-\frac 1p-\frac
  1r})(\|\bX\|_{r, +\infty} + \e)}{\sqrt m} \right) \left( 1+\sqrt{2d(n+1) \log(6\sqrt m)}\right),
\end{align*}
which completes the proof.
\end{proof}

\section{Characterizing adversarial perturbations for ReLU neural networks}

\subsection{Condition for adversarial perturbations to be on the $r$-sphere~(proof of Theorem~\ref{th:unit_norm_at_opt})}
\label{app:unit_norm_at_opt}
In this section we provide the proof of Theorem~\ref{th:unit_norm_at_opt} which characterizes adversarial perturbations to a one-layer neural net. First, by the extreme value theorem, \eqref{eq:objective} achieves
  its minimum on $\|\bs\|_r\leq 1$. Thus we can restate \eqref{eq:objective} as  
\begin{align}
\label{eq:app_objective}
& \min_{\|\bs\|_r\leq 1} f(\bs) = \sum_{j = 1}^n
  u_j (\bw_j\cdot (\bx+\e \bs))_+ .
\end{align}
\begin{reptheorem}{th:unit_norm_at_opt}
Let $d$ be the dimension and $n$ the number of neurons. Consider \eqref{eq:app_objective} as defined above. If either  $\|\bx\|_r\geq \e$ or $n<d$, an optimum is attained on the sphere $\set{\bs \colon \|\bs\|_r = 1}$.
Otherwise, an optimum is attained either at $\bs=-\frac 1\e \bx$ or on $\|\bs\|_r=1$.
\end{reptheorem}
The proof of this theorem relies on two important lemmas stated below. We defer the proofs of these lemmas to the end of the section.



\begin{lemma}
\label{lemma:induction}
Consider (\ref{eq:app_objective}). Then an optimum is obtained in either
\begin{enumerate}
\item $S_1:=\{\bs\colon \|\bs\|_r=1\}$
\item $S_2=\{\bs\colon\bw_{j_k}\cdot (\bx+\e\bs)=0\text{ for linearly independent }\bw_{j_1}\ldots \bw_{j_d}\}$
\end{enumerate}
\end{lemma}
\begin{lemma}
\label{lemma:intersection}
Consider the intersection of $d$ linearly independent hyperplanes defined by
\begin{equation}\label{eq:lin_indep}
\bv_k\cdot (\bx+\e\bs)=0\colon k=1\ldots d
\end{equation}
for a fixed $\bx$.  They intersect at a single point given by $\bs=-\frac 1\eps \bx$.
\end{lemma}


Next we use lemmas~\ref{lemma:induction} and~\ref{lemma:intersection}
to prove Theorem~\ref{th:unit_norm_at_opt}.
\begin{proof}[Proof of Theorem~\ref{th:unit_norm_at_opt}]
   By Lemma~\ref{lemma:induction},
  there exists a point $\bs^*$ with
\[
f(\bs^*)=\min_{\|\bs\|_r\leq 1}f(\bs)\] for which either
$\|\bs^*\|_r=1$ or \[\{\bs^*\colon\bw_{j_k}\cdot (\bx+\e\bs^*)=0\text{
    for some linearly independent }\bw_{j_1}\ldots \bw_{j_d}\}
\]
If $n<d$, then there aren't $d$ linearly independent $w_i$s, and thus $\bs^*$ satisfies 
$\|\bs^*\|_r=1$.

Now assume that $n\geq d$ and $\|\bs^*\|_r\neq
1$. Lemma~\ref{lemma:intersection} implies that
$\bs^*=-\frac 1 \e \bx$ and hence $\|\bx\|_r<\e$. Taking the
contrapositive of this statement results in
\[n\geq d\text{ and }\|\bx\|_r\geq \e\Rightarrow \|\bs^*\|_r=1\]
\end{proof}

We end the subsection with the proofs of lemmas~\ref{lemma:induction} and~\ref{lemma:intersection}. Before we prove Lemma~\ref{lemma:induction} we state and prove a simpler statement that will be used in its proof.
\begin{lemma}
\label{lemma:zero-gradient}
Consider (\ref{eq:app_objective}). Then an optimum is obtained at either
\begin{enumerate}
\item $S_1:=\{\bs\colon \|\bs\|_r=1\}$
\item $S_2=\{\bs\colon\bw_j\cdot (\bx+\e\bs)=0\text{ for some }\bw_j\}$
\end{enumerate}
\end{lemma}
\begin{proof}
  We know from calculus that every extreme point of $f$ is obtained
  either on the boundary of the optimization region, at a point where
  the function isn't differentiable, or where the derivative is zero. First, observe that at any non-differentiable point with $\|\bs\|_r<1$, some $\bw_j$ must satisfy $\bw_j\cdot (\bx+\e\bs)=0$.
  Now we'll consider the third case, points where $\nabla f(\bs)=0$. Assume that $\bs^*$ is an extreme point for which $f$ is
  differentiable (and with derivative zero). Then we claim that there
  is another point in either $S_1$ or $S_2$ that achieves the same
  objective value.  Let $P=\{j: \bw_j\cdot (\bx+\e \bs^*)>0\}$ Then
  \[f(\bs^*)=\sum_{j\in P} u_j (\bw_j\cdot (\bx+\e \bs^*))\] Fix this
  set $P$.  Note that the region where
  \[f(\bs)=\sum_{j\in P} u_j (\bw_j\cdot (\bx+\e \bs))\] is
  defined by
\begin{align}
R =
\left\{\bs\colon
\|\bs\|_r\leq 1,
\bw_j\cdot (\bx+\e \bs)\geq 0\text{ for } j\in P,
\bw_j\cdot (\bx+\e \bs)\leq 0\text{ for } j\in P^C\right\}
\label{eq:region-1}
\end{align}
By assumption,
\[
\nabla f(\bs^*)=\e \sum_{j\in P} u_j\bw_j=0
\] 
However, for any other $\bs$ in the region defined by (\ref{eq:region-1})
\[
\nabla f(\bs)=\e \sum_{j\in P} u_j\bw_j=f(\bs^*)=0
\]
Hence, $f$ is constant on the interior of the region defined by
(\ref{eq:region-1}). By continuity, it is constant on the closure of
this region as well.  Hence an optimum of the same value is obtained
in either $S_1$ or $S_2$.
\end{proof}

\begin{proof}[Proof of Lemma~\ref{lemma:induction}]
  This will be a proof by induction. Let $\bs^*$ be an
  optimum. Define
  $Z_0^{\bs}=\{\bw_j:\bw_j\cdot (\bx+\e \bs)=0\}$ and let $k$ be the
  dimension of $\spn(Z_0^{\bs^*})$. The induction will be on $k$.
\paragraph{Base Case:}
By the previous lemma, when looking for the optimum, we only need to
consider $\bs$ for which $\|\bs\|_r=1$ or $\bw_j\cdot (\bx+\e \bs)=0$
for some $j$. Assume that we have an extreme point $\bs^*$ for which
$\|\bs^*\|<1$. Then $k\geq 1$.
\paragraph{Inductive Step:} Let $\bs^*$ be our extreme point and assume
that $\|\bs^*\|_r<1$.  Our induction hypothesis is that
$\dim (\spn( Z_0^{\bs^*}))=k<d$. We will show that there is another
point $\bt$ that achieves the same objective value satisfying either
$\|\bt\|_r=1$ or $\dim(\spn(Z^\bt_0))=k+1$.

Let $Z$ be any linearly independent subset of $Z_0^{\bs^*}$. We can
parameterize $\bs$ to be in the intersection of the hyperplanes that
define $Z$. Formally, let $\bv\in \spn(Z)$ with
$\bw_j\cdot (\bx+\e \bv)=0$ for all $\bw_j\in Z$, and let
$\bA\colon \mathbb R^{d-k}\to \mathbb R^d$ be a matrix whose columns
span $Z^\perp$. Take $\bs=\bv+\bA\bs'$,
$P=\{j\colon \bw_j\cdot (\bx+\e \bs^*)>0\},$ and
$ N=\{j\colon \bw_j\cdot (\bx+\e \bs^*)<0\}$.  Then by continuity,
\[f(\bs)=\sum_{j\in P} u_j \bw_j\cdot (\bx+\e (\bv+\bA\bs'))\] holds
on the region defined by
\begin{align}R=\{\bs'\colon 
\|\bv+\bA \bs'\|_r\leq 1,
\bw_j\cdot (\bx+\e (\bv+\bA\bs'))\geq 0\text{ for } j\in P,
\bw_j\cdot (\bx+\e (\bv+\bA\bs'))\leq 0\text{ for } j\in N\}
\label{eq:region-2}
\end{align}
For convenience, set 
\[
g(\bs')\colon = f(\bv+\bA \bs')
\] 
We assumed that our optimum $\bs^*$ satisfied $\|\bs^*\|_r< 1$ and
$\bw_j\cdot (x+\e \bs)\neq 0$ for $j\in P\cup N$, which entails that our
critical point is in the interior of $R$.  On the interior of this
region, to find all critical points, we can differentiate $g$ in
$\bs'$:
\[
\nabla g(\bs')=\bA^T \sum_{j\in P} u_j\bw_j
\] 
and set $\nabla g(\bs')$ equal to zero.
This expression is independent of $\bs'\in R$.  Let $\bz$ be a
critical point of $g$ in $\interior (R)$. Then $\nabla g(\bz)=0$
implies that $\nabla g(\bs')=0$ for all $\bs'\in \interior(R)$. Hence,
$g$ is constant on $R$. This implies that there is another point
$\bs'$ with the same objective value on $\partial R$. For this point, either
$\|\bv+\bA\bs'\|_r= 1$, or $\|\bv+\bA\bs'\|_r<1$ and
$\bw_j\cdot (\bx+\e (\bv+\bA\bs'))=0$ for some $j\in P\cup N$. If the
second option holds, $j\in P\cup N$ means that
$\bw_j\not \in \spn Z_0^{\bs^*}$. It follows that
$\spn( Z_0^{\bs^*}\cup \{\bw_j\})$ is dimension $k+1$ and this
completes the induction step.
\end{proof}

Finally we prove Lemma~\ref{lemma:intersection}.
\begin{proof}[Proof of Lemma~\ref{lemma:intersection}]
  By substitution $\bs=-\frac 1\eps \bx$ is a solution to the
  system of equations \eqref{eq:lin_indep}. Since $d$ linearly independent equations intersect at a
  point, it is the only solution to these equations.
\end{proof}

\subsection{A Necessary Condition}
\label{app:necessary_condition}

In this subsection we present a necessary condition at the optimum when perturbations are measured in any general $r$-norm.  Throughout this subsection, $\bu \astrosun \bv$ will be the elementwise product of $\bu$
an $\bv$, $\bu^r$ will be elementwise exponentiation, $\|\bv\|$ will be elementwise absolute value, and $\sgn(\bv)$
will be the vector of signs of the components of $\bv$. We adopt the convention $\sgn(0)=0$. Recall the definition
of dual norm:
\[
\|\bu\|_{\du r}=\sup_{ \|\bv\|_r\leq 1} \bu\cdot \bv=\|\bu\|_{\du
    r}
\] 
Equality holds at the vector
$\bv = \frac 1 {\|\bu\|_r^{r-1}} |\bu|^{r-1}\astrosun\sgn(\bu)$, which has 
unit $r^*$-norm. For convenience we, define
\[ \dual_r(\bu)=(\sgn \bu)\astrosun \frac{|\bu|^{r-1}}{\|\bu\|_r^{r-1}}\] which gives
\[\bu\cdot \dual_r(\bu)=\|\bu\|_r^r=1.\]

Below we state and prove the main theorem of this section.

\begin{theorem}
\label{th:necessary_opt}
Let $1<r<\infty$. 
Take 
\begin{equation}
\label{eq:obj_appendix} 
f(\bs)
=\sum_{j=1}^n u_j(\bw_j\cdot (\bx+\e \bs))_+
\end{equation} 
Assume that either $\|\bx\|_r \geq \e$ or $n<d$. Let $\bs^*$ is a
minimizer of $f$ on the unit $r$-sphere. Define the following sets:
\[
P=\{j\colon \bw_j\cdot (\bx+\e \bs^*)> 0\}
\]
\[
Z=\{j\colon \bw_j\cdot (\bx+\e \bs^*)=0\}
\]
\[
N=\{j\colon \bw_j\cdot (\bx+\e \bs^*)<0\}
\]

Let $P_Z$ be the orthogonal projection onto the subspace spanned by
the vectors in $Z$, and $P_{Z^C}$ be the projection onto the complement
of this subspace. 
Then the following holds:
If $P\neq \emptyset$
\begin{equation}
\label{eq:s_star-v1}
\bs^*=-\frac \e \lambda\left|  \left(\sum_{j\in P} u_j\bw_j 
+ \sum_{j\in Z} t_ju_j\bw_j\right)\right|^{ r-1}\astrosun \sgn
\left(\sum_{j\in P} u_j\bw_j 
+ \sum_{j\in Z} t_ju_j\bw_j\right)
\end{equation}
where the constants $t_j$, $\lambda$ are given by the equations 
\begin{align}
& \|\bs^*\|_r=1\label{eq:small_constraint}\\
& P_Z\bs^*=-\frac 1\e P_Z\bx\label{eq:big_constraint}
\end{align} 

Further, if $P=\emptyset$,
\begin{equation}\label{eq:P-empty}\bs^*=-\frac{P_Z \bx}{\|P_Z\bx\|}\end{equation}
Using the $\dual_r$ notation, $\bs^*$ can be expressed as
\begin{equation}
\bs^*
=\dual_{r}\left(  \left|\sum_{j\in P} u_j\bw_j 
+ \sum_{j\in Z} t_ju_j\bw_j\right|\right)\astrosun \sgn(\sum_{j\in P}
u_j\bw_j 
+\sum_{j\in Z} t_ju_j\bw_j)
\end{equation}
\end{theorem} 

Notice that for $r=2$, $\dual_{r}(\bs^*)=\bs^*$ and then we can write $\bs^*$ explicitly:
\[
\bs^*=-\left(\sqrt{1-\frac{\|P_Z\bx\|_2^2}{\e^2}}\frac{P_{Z^C}\sum_{j\in P}
  u_j \bw_j}{\left\| P_{Z^C}\sum_{j\in P} u_j \bw_j\right\|_2}+\frac
{\|P_Z \bx\|_2}{\e} \frac{P_Z\bx}{\|P_Z \bx\|}\right)
\]

Before proceeding with the proof of this theorem, we state a useful definition and lemma. Recall the definition of the subgradient of a convex function:
\begin{definition}
The subdifferential of a convex function $f_1$ is the set
\[
\partial f_1(\bx)=\{\bv\colon  f_1(\mathbf y)-f_1(\bx)\geq \bv\cdot (\mathbf
y-\bx)\}
\]
while the subdifferential of a concave function $f_2$ is the set
\[-\partial (-f_2(\bx))=\{\bv: f_2(\by)-f_2(\bx)\leq \bv\cdot (\by-\bx)\} \]
\end{definition}
For a function $f=f_1+f_2$ that is
the sum of a convex function $f_1$ and a concave function $f_2$, the following observation from \cite{Polyakova1984} shows why these definitions are useful for us.

\begin{lemma}
\label{lemma:cvx_cnv_sum}
\label{cor:sbgd_cvx_plus_concave}
Let $f=f_1+f_2$ with $f_1$ convex and $f_2$ concave. Assume that $f$ has a local minimum at $x^*$.
Then
\[
\mathbf 0\in \partial f_1(\bx^*)+\partial f_2(\bx^*)
\]
\end{lemma}
Note that the same statement holds for local maxima of $f$. 
We defer the proof of this lemma to the end of this subsection.

To prove Theorem~\ref{th:necessary_opt}, we form a Lagrangian for
computing the optimum of
~\eqref{eq:obj_appendix}. Lemma~\ref{lemma:cvx_cnv_sum} gives a
necessary condition in terms of the subgradient of this
Lagrangian. Subsequently, we use information about the dual variables
obtained via Theorem~\ref{th:unit_norm_at_opt} and convexity to show
\eqref{eq:s_star-v1}, \eqref{eq:small_constraint}, and \eqref{eq:P-empty}. (Note that
either $\|\bx\|_r\geq \e$ or $n<d$ are precisely the conditions for
Theorem~\ref{th:unit_norm_at_opt}). After that, standard linear algebra
shows \eqref{eq:big_constraint}.  

\begin{proof}[Proof of Theorem~\ref{th:necessary_opt}]

\noindent \textbf{Establishing Equations \eqref{eq:s_star-v1} and \eqref{eq:small_constraint}:}
First note that the objective $f$ is the sum of a convex and a concave
function: take
\begin{align*}
f_1(\bs) = \sum_{j\colon u_j>0} u_j\left(\bw_j\cdot (\bx+\e \bs)\right)_+\; 
f_2(\bs) = \sum_{j\colon u_j<0} u_j\left(\bw_j\cdot (\bx+\e \bs)\right)_+
\end{align*}
$f_1$ is convex because it is the sum of convex functions and $f_2$ is
concave because it is the sum of concave functions. This observation
will allow us the apply Lemma~\ref{cor:sbgd_cvx_plus_concave}. We form the corresponding Lagrangian: 
\[
L(\bs)= \sum_{j=1}^n u_j(\bw_j\cdot (\bx+\e \bs))_++\frac \lambda
r(\|\bs\|^r_r-1)
\]
$L$ is convex in an open set around every local minimum. On this set, since we are optimizing over $\|\bs\|_r\leq 1$, we know that
$\lambda \geq 0$. Further, Theorem~\ref{th:unit_norm_at_opt} shows that there must be an optimum
on the unit $r$-sphere for $\|\bx\|_r\geq \e$.

By Lemma~\ref{cor:sbgd_cvx_plus_concave}, we want to find a condition when $\mathbf 0$ is in the subdifferential.  We use the following
two facts:
\begin{enumerate}
\item 
\[
\partial (x)_+ = 
\begin{cases}
\{0\}&\text{ if }x<0\\
[0,1]&\text{ if }x=0\\
\{1\}&\text{ if }x>0
\end{cases}
\]
\item For $1<r<\infty$, the $r$ norm is differentiable. 
  Hence we can write:
\[
\nabla \|\bs\|_r^r=|\bs|^{r-1}\astrosun
\sgn\bs=\|\bs\|_{r}^{r-1}\dual_{r^*}(\bs)
\] 
Hence, if $\|\bs\|=1$,
$\partial \|\bs\|_r^r=\dual_{r^*}(\bs)=\sgn \bs\astrosun |\bs|^{r-1}$.
\end{enumerate}

Then applying Lemma~\ref{cor:sbgd_cvx_plus_concave}, we need
\begin{align*}
\mathbf 0\in \e\partial\sum_{j\in P} u_j(\bw_j\cdot (\bx+\e \bs)_++\e\partial \sum_{j\in Z} u_j(\bw_j\cdot (\bx+\e \bs))_++\e\partial \sum_{j\in N} u_j(\bw_j\cdot (\bx+\e \bs)_+)+\partial \frac \lambda r(\|\bs\|_r^r-1).
\end{align*}
Hence for some $t_j\in[0,1]$, 
\begin{align}
\mathbf 0=\e\sum_{j\in P} u_j \bw_j+\e\sum_{j\in Z} t_j u_j\bw_j+\frac \lambda r\partial \|\bs\|_r^r\label{eq:dus}
\end{align}
Using Theorem~\ref{th:unit_norm_at_opt}, we choose an optimum on the
boundary $\|\bs\|_r=1$. First we consider $\bs^*$ with $\lambda\neq 0$.
This allows for solving for $\partial \|\bs\|_r^r$:
\[
\dual_{r}(\bs^*)=\bs^* \astrosun |\bs^*|^{r-1}=-\frac \e\lambda \left(
  \sum_{u\in P} u_j \bw_j+\sum_{j\in Z} t_j u_j\bw_j\right)
\]
Now since $\dual_{r} (\bs^*)$ has $\du r$-norm 1, this allows us to
solve for $|\lambda|$. Further recall that at a local minimum, $\lambda\geq 0$ which tells us $\sgn \lambda$. Using this information, we can solve for $\lambda$ which establishes~\eqref{eq:small_constraint}.  Since
$1<r<\infty$, this equation further establishes~\eqref{eq:s_star-v1}.

\paragraph{Establishing Equation \eqref{eq:P-empty}:}
Now we consider the case where $\lambda=0$ or $P=\emptyset$. For $\lambda=0$, we will show by contradiction that $P$ must be empty. Assume that $P\neq \emptyset$. Equation \eqref{eq:dus} then simplifies to 
\[\mathbf 0=\e \sum_{j\in P} u_j \bw_j+\e \sum_{u\in Z} t_j u_j\bw_j\] which implies that 
\[\sum_{j\in P} u_j\bw_j=-\sum_{j\in Z} t_j u_j \bw_j\]
However, if we take the dot product with $\bx+\e \bs^*$,
\[\sum_{j\in P} u_j \bw_j\cdot (\bx+\e\bs^*)=-\sum_{j\in Z} \bw_j\cdot (\bx+\e \bs^*)=0\] and therefore, $\bw_j\cdot (\bx+\e\bs^*)\leq 0$ for some $j\in P$ which contradicts the definition of $P$. Therefore, $P$ must be empty. 

Now we assume that $\bs^*$ has $P=\emptyset$ and we show that there is a point $\bz^*$ that achieves the same objective value as $\bs^*$ but has $N=\emptyset$. This will be proved by induction on the size of $N_\bz$. This will then imply that we can take $\bs^*=-\frac{P_Z \bx}{\|P_Z\bx\|}$.

Denote by $Z_{\bs},N_{\bs}$
\[P_\bs=\{j:\bw_j\cdot (\bx+\e\bs)>0\}\]
\[Z_\bs=\{j:\bw_j\cdot (\bx+\e\bs)=0\}\]
\[N_\bs=\{j: \bw_j\cdot (\bx+\e\bs)<0\}\]
For the base case, we use a point $\bs$ that achieves the optimal value and has $P_\bs=\emptyset$. If $N_{\bs}=\emptyset$, we are done. Otherwise, for the induction step, we assume $N_\bs\neq \emptyset$. We will find a vector $\bz$ that achieves these same objective value as $\bs$, but $N_{\bs}\supsetneq N_{\bz}$. Pick a vector $\bv$ perpendicular to $\spn\{\bw_j\}_{j\in Z_\bs}$ but not perpendicular to $\spn\{\bw_j\}_{j\in N_s}$. Such a vector must exist because if $\bw_k\in \spn\{\bw_j\}_{j\in Z_\bs}$, then $\bw_k\in \bZ_s$. We now consider 
\[\bz(\delta)=\frac{\bs+\delta \bv}{\| \bs+\delta \bv\|}\]
Note that 
\[\bz(\delta)\cdot \bw_j=0\] for each $j\in Z_\bs$ for all $\delta$. Because the strict inequality
\[\bw_j\cdot (\bx+\e\bz(\delta))<0\; j\in N_\bs\] is satisfied for $\delta=0$, it is also satisfied for some small $\delta\neq 0$. We can now increase or decrease $\delta$ until
\[\bw_j\cdot (\bx+\e\bz(\delta))=0\;\text{ for some } j\in N_\bs\] and $\bw_j\cdot (\bx+\e\bz(\delta))<0$ for the remaining $j$s in $N$. We then have $N_{\bs}\supsetneq N_{\bz(\delta)}$. Furthermore, $f(\bs)=f(\bz(\delta))$ because the set $P$ is still empty.

\noindent \textbf{Establishing Equation \eqref{eq:big_constraint}:}
Let $\{\mathbf f_k\}_{k=1}^{d_Z}$ be an orthonormal basis of
$\spn\{\bw_j\}_{j\in Z}$. We will show that
$\bx\cdot \mathbf f_k=-\e \bs^*\cdot \mathbf f_k$. Since $P_Z \bx$ and
$-\e P_Z\bs^*$ are contained in the subspace spanned by the vectors in
$Z$, this would imply that $P_Z \bs^*=-\frac 1 \e P_Z \bx$.  Let
\begin{equation}
\label{eq:project-basis} \mathbf f_k=\sum_{j\in Z} a_{kj}\bw_j
\end{equation} 
for some constants $a_{kj}$. Recall that for all $j\in Z$, 
\[
\bw_j\cdot (\bx+\e \bs^*)=0.
\]
We then use the above equation and (\ref{eq:project-basis}) to take the dot
product of $\bx$ and $\mathbf f_k$:
\[ 
\bx \cdot \mathbf f_k= \bx\cdot \sum_{j\in Z} a_{kj} \bw_j= \sum_{j\in
  Z} a_{kj}\bx\cdot \bw_j=-\e \sum_{j\in Z} a_{kj}\bs^*\cdot \bw_j
=-\e  \mathbf f_k \cdot \bs^*.
\]
The above establishes equation \eqref{eq:big_constraint} and completes the proof of the theorem.
\end{proof}

We end the section by proving Lemma~\ref{lemma:cvx_cnv_sum}.

\begin{proof}[Proof of Lemma~\ref{lemma:cvx_cnv_sum}]
We will show that
\begin{equation}
\label{eq:containment}
-\partial f_2(\bx^*) \subset \partial f_1(\bx^*)
\end{equation}
This implies 
\[
\mathbf 0\in \partial f_1(\bx^*)+\partial f_2(\bx^*).
\]

We prove \eqref{eq:containment} by contrapositive. We pick a point $\bx^*$ and assume that \eqref{eq:containment} does not hold. Then we show that $\bx^*$ cannot be a minimum.
Assume \eqref{eq:containment} does not hold. This assumption implies that for
some vector $\bc$, $\bc\in \partial f_2(\bx^*)$ but
$-\bc\not \in \partial f_1(\bx^*)$. Then there exists an $\bx$ for
which
\[
f_2(\bx)-f_2(\bx^*)\leq \bc^T(\bx-\bx^*)
\]
\[
f_1(\bx)-f_1(\bx^*)< -\bc^T(\bx-\bx^*)
\] Summing the above inequalities, we get: 
\[
f_1(\bx)+f_2(\bx)< f_1(\bx^*)+f_2(\bx^*)
\] 
so $\bx^*$ cannot be a local minimum.
\end{proof}
\section{Towards Dimension-Independent Bounds for Neural Networks}
\subsection{Proof of Theorem~\ref{th:one_layer_nn_rc_shatter}}
Recall from Section~\ref{sec:dim-independent-neural-nets} that given a sample $\cS$, $C_\cS$ denotes the set of all possible partitions of points in $\cS$ that can be obtained based on the sign pattern they induced over the set of weight vectors $\bu, \bw_1, \bw_2, \dots, \bw_n$. For a given partition $\cC \in \cC_\cS$, we denote by $n_\cC$ the number of parts in $\cC$. Furthermore, we define $C^*_\cS$ to be the size of the set $C_\cS$ and $\Pi^*_\cS = \max_\cC n_\cC$. We now proceed to prove Theorem~\ref{th:one_layer_nn_rc_shatter} that establishes a data dependent bound on the Rademacher complexity of neural networks with one hidden layer. 
\begin{reptheorem}{th:one_layer_nn_rc_shatter}
Consider the family of functions $\cG^n_p$ with $p \in [1,\infty]$, activation function $\rho(z)=(z)_+$,
and perturbations in $r$-norm for $1<r<\infty$. Assume that for our sample $\|\bx_i\|_r\geq \e$.
Then, the following upper bound on the Rademacher complexity holds:
\[
\wt \R_\cS(\cG_p^n)
\!\leq\! \bigg[\! \frac{W\Lambda \max(1,d^{1-\frac 1p-\frac 1r})(K(p,d)\|\bX^\top\|_{\infty, p^*}+\e)}{\sqrt m} \!\bigg] C^*_\cS \sqrt{\Pi^*_\cS},
\]
where $K(p,d)$ is defined as
\begin{equation}\label{eq:def_K(p,d)}
K(p,d)=
\begin{cases}
\sqrt {2\log(2d)}&\text{ if } p=1\\
\sqrt 2 \left[ \frac{\Gamma(\frac{\du p+1} 2)}{\sqrt \pi}\right]^\frac 1 {\du p} &\text{ if } 1 < p \leq 2\\
1&\text{ if } p \geq 2\\
\end{cases}
\end{equation}
\end{reptheorem}
\begin{proof}[Proof of Theorem~\ref{th:one_layer_nn_rc_shatter}]
   Let $\cC_t$ denote a partition in partitions $\cC$. Furthermore, define $\bs_t =\argmin_{\|\bs\|_r\leq 1} \sum_{j=1}^n u_j\bw_j\cdot (\bx+\e \bs)_+$ for $\bx\in \cC_t$ and $P_t=\{j: \bw_j\cdot (\bx+\e \bs_t)>0\}$. The Rademacher complexity of the network can be bounded as
\begin{align}
\wt \R_\cS(\cG_p^n)
& =\E_\bsig\Bigg[\sup_{\substack{\|\bw_j\|_p\leq W\\ \|\bu\|_1\leq \Lambda}}\frac 1m  \sum_{i=1}^m \sigma_i\inf_{\|\bs\|_r\leq 1} y_i\sum_{j=1}^n u_j(\bw_j\cdot (\bx_i+\e\bs))_+\Bigg]\nonumber\\
& =\E_\bsig\Bigg[\sup_{\substack{\|\bw_j\|_p\leq W\\ \|\bu\|_1\leq \Lambda}}\frac 1m  \sum_{i=1}^m \sigma_i y_i\sum_{j=1}^n u_j(\bw_j\cdot (\bx_i+\e\bs_i))_+\Bigg]&(\text{definition of }\bs_i)\nonumber\\
& =\E_\bsig\Bigg[\sup_{\substack{\|\bw_j\|_p\leq W\\ \|\bu\|_1\leq \Lambda}}\frac 1m  \sum_{t=1}^{n_\cC}\sum_{i\in \cC_t} \sigma_i y_i\sum_{j=1}^n u_j(\bw_j\cdot (\bx_i+\e\bs_t))_+\Bigg] & (\text{definition of }\cC_t)\nonumber\\
& =\E_\bsig\Bigg[\sup_{\substack{\|\bw_j\|_p\leq W\\ \|\bu\|_1\leq \Lambda}}\frac 1m  \sum_{t=1}^{n_\cC}\sum_{i\in \cC_t} \sigma_i y_i\sum_{j\in P_t} u_j(\bw_j\cdot (\bx_i+\e\bs_t))\Bigg] & (\text{definition of }P_t)\nonumber\\
& \leq \left( \E_\bsig\Bigg[\sup_{\substack{\|\bw_j\|_p\leq W\\ \|\bu\|_1\leq \Lambda}}\frac 1m \sum_{t=1}^{n_\cC} \sum_{i\in \cC_t} \sigma_i y_i\sum_{j\in P_t} u_j\bw_j\cdot\bx_i\Bigg]+\E_\bsig\Bigg[\sup_{\substack{\|\bw_j\|_p\leq W\\ \|\bu\|_1\leq \Lambda}}\frac 1m \sum_{t=1}^{n_\cC} \sum_{i\in \cC_t} \sigma_i y_i\sum_{j\in P_t}\e u_j\bw_j\cdot \bs_t))\Bigg] \right) \label{eq:adv_rc_sum}
\end{align}
Next we bound each term in equation (\ref{eq:adv_rc_sum}) separately. For the first term we can write:
\begin{align}
\E_\bsig\Bigg[\sup_{\substack{\|\bw_j\|_p\leq W\\ \|\bu\|_1\leq \Lambda}}\frac 1m  \sum_{t=1}^{n_\cC}\sum_{i\in \cC_t} \sigma_i y_i\sum_{j\in P_t} u_j\bw_j\cdot\bx_i\Bigg]
&=\frac 12  \E_\bsig\Bigg[\sup_{\substack{\|\bw_j\|_p\leq W\\ \|\bu\|_1\leq \Lambda}}\left|\frac 1m\sum_{t=1}^{n_\cC}  \sum_{i\in \cC_t} \sigma_i y_i\sum_{j\in P_t} u_j\bw_j\cdot\bx_i\right|\Bigg]&\text{(sign symmetry)}\nonumber\\
&=\frac 12  \E_\bsig\Bigg[\sup_{\substack{\|\bw_j\|_p\leq W\\ \|\bu\|_1\leq \Lambda}}\left|\frac 1m \sum_{t=1}^{n_\cC} \sum_{j\in P_t}u_j\bw_j\cdot\sum_{i\in \cC_t} \sigma_i y_i \bx_i\right|\Bigg]&\text{(reordering summations)}\nonumber\\
&\leq \frac W2  \E_\bsig\Bigg[\sup_{\substack{\|\bw_j\|_p\leq W\\ \|\bu\|_1\leq \Lambda}}\frac 1m \sum_{t=1}^{n_\cC} \sum_{j\in P_t}|u_j|\left\|\sum_{i\in \cC_t} \sigma_i y_i \bx_i\right\|_{\du p}\Bigg]&\text{(dual norm definition)} \nonumber 
\end{align}
Using the bound on the $\ell_1$ norm of $\bu$ we get:
\begin{align}
\E_\bsig\Bigg[\sup_{\substack{\|\bw_j\|_p\leq W\\ \|\bu\|_1\leq \Lambda}}\frac 1m  \sum_{t=1}^{n_\cC}\sum_{i\in \cC_t} \sigma_i y_i\sum_{j\in P_t} u_j\bw_j\cdot\bx_i\Bigg]
&\leq \frac W2  \E_\bsig\Bigg[\sup_{\substack{ \bW, \bu}}\frac 1m \sum_{t=1}^{n_\cC} \Lambda\left\|\sum_{i\in \cC_t} \sigma_i y_i \bx_i\right\|_{\du p}\Bigg]&\text{(dual norm definition)}\nonumber\\
&\leq \frac 1m  \frac {\Lambda W}2  \E_\bsig\Bigg[\sup_{\substack{ \bW, \bu}}\sum_{t=1}^{n_\cC}\left\|\sum_{i\in \cC_t} \sigma_i \bx_i\right\|_{\du p}\Bigg]&(\sigma_i \text{ distributed like }y_i\sigma_i)\nonumber \\
&\leq \frac 1m  \frac {\Lambda W}2  \E_\bsig\Bigg[ \sum_{\mathcal{C}} \sum_{t=1}^{n_{\mathcal{C}}}\left\|\sum_{i\in \cC_t} \sigma_i \bx_i\right\|_{\du p} \Bigg]&\text{(summing over all partitions)}\nonumber \\
&= \frac 1m  \frac {\Lambda W}2   \sum_{\mathcal{C}} \sum_{t=1}^{n_{\mathcal{C}}}\E_\bsig\Bigg[\left\|\sum_{i\in \cC_t} \sigma_i \bx_i\right\|_{\du p} \Bigg].\nonumber
\end{align}

Next, note that 
\[\E_\bsig\Bigg[\left\|\sum_{i\in \cC_t} \sigma_i \bx_i\right\|_{\du p} \Bigg] =\E_\bsig\Bigg[\sup_{\|\bw\|_p\leq 1}\sum_{i\in \cC_t} \sigma_i \bw\cdot \bx_i \Bigg] =|\cC_t|\R_{\cC_t}(\cF_p)\]
where $\cF_p$ is the linear function class defined in \eqref{eq:linear_function_class} with $W=1$.
Hence, applying Theorem~\ref{th:linear_rc},
\begin{align}
\label{eq:X_t_top}
    \E_\bsig\Bigg[\left\|\sum_{i\in \cC_t} \sigma_i \bx_i\right\|_{\du p} \Bigg] &\leq K(p,d)\|\bX^\top_t\|_{2,p^*}
\end{align}
with $K(p,d)$ as defined in \eqref{eq:def_K(p,d)}.
$\bX_t$ is the matrix with data points in $\cC_t$ as columns. Furthermore, we can write:
\begin{align*}
    \|\bX^\top_t\|_{2,p^*} &= \bigg(\sum_{j=1}^d \|\bX_t(j)\|_2^{\du p} \bigg)^{\frac{1}{p^*}}\,\, \text{[$\bX_t(j)$ denotes $j$th row of $\bX$]}\\
    &\leq \sqrt{|\cC_t|}\bigg(\sum_{j=1}^d \|\bX(j)\|_\infty^{p^*} \bigg)^{\frac{1}{p^*}} \\
    &= \sqrt{|\cC_t|}\|\bX^\top\|_{\infty, p^*}.
\end{align*}
Using the above bound we can write:
\begin{align}
\E_\bsig\Bigg[\sup_{\substack{\|\bw_j\|_p\leq W\\ \|\bu\|_1\leq \Lambda}}\frac 1m  \sum_{t=1}^{n_\cC}\sum_{i\in \cC_t} \sigma_i y_i\sum_{j\in P_t} u_j\bw_j\cdot\bx_i\Bigg]&\leq   \frac {K(p,d)\Lambda W}m   \sum_{\mathcal{C}} \sum_{t=1}^{n_{\mathcal{C}}} \sqrt{|C_t|} \|\bX^\top\|_{\infty,p^*} \nonumber \\
&\leq \frac {K(p,d)\Lambda W}{\sqrt{m}} |\mathcal{C}_{\cS}^*| \sqrt{\Pi_\cS^*} \|\bX^\top\|_{\infty,p^*}.\label{eq:nn-dim-ind-1}
\end{align}
Here the last inequality follows from the fact that $\sum_{t=1}^{n_\cC} |\cC_t| = m$ and $\sum_{t=1}^{n_\cC} \sqrt{|\cC_t|}$ is maximized when $|\cC_t| = m/n_\cC$ for all $t$.
Now for the second term in (\ref{eq:adv_rc_sum}) we can write:
\begin{align}
\E_\bsig\Bigg[\sup_{\substack{\|\bw_j\|_p\leq W\\ \|\bu\|_1\leq \Lambda}}\frac 1m \sum_{t=1}^{n_\cC} \sum_{i\in \cC_t} \sigma_i y_i\sum_{j\in P_t}\e u_j\bw_j\cdot \bs_t\Bigg] &=\E_\bsig\Bigg[\sup_{\substack{\|\bw_j\|_p\leq W\\ \|\bu\|_1\leq \Lambda}}\frac 1m \sum_{t=1}^{n_\cC} \sum_{i\in \cC_t} \sigma_i \sum_{j\in P_t}\e u_j\bw_j\cdot \bs_t\Bigg] &(y_i\sigma_i\text{ distributed like }\sigma_i)\nonumber\\
&=\E_\bsig\Bigg[\sup_{\substack{\|\bw_j\|_p\leq W\\ \|\bu\|_1\leq \Lambda}}\frac 1m \sum_{t=1}^{n_\cC}\sum_{j\in P_t} \e u_j\bw_j\sum_{i\in \cC_t} \sigma_i \cdot \bs_t\Bigg] &\text{(reorder summations)}\nonumber\\
&\leq \E_\bsig\Bigg[\sup_{\substack{\|\bw_j\|_p\leq W\\ \|\bu\|_1\leq \Lambda}}\frac 1m \sum_{t=1}^{n_\cC}\sum_{j\in P_t} \e |u_j|W\left\|\sum_{i\in \cC_t} \sigma_i \cdot \bs_t\right\|_{\du p}\Bigg]&(\text{dual norm}) \nonumber\\
&\leq \frac {\epsilon W\Lambda} m \E_\bsig \Bigg[\sup_{\substack{\|\bw_j\|_p\leq W\\ \|\bu\|_1\leq \Lambda}}\sum_{t=1}^{n_\cC}\left\|\sum_{i\in \cC_t} \sigma_i \cdot \bs_t\right\|_{\du p}\Bigg]&\text{(dual norm)} \nonumber \\
&\leq \frac {\epsilon W\Lambda} m\sup_{\|\bs_t\|_{\du r} \leq 1}\|\bs_t\|_{\du p} \E_\bsig \Big[\sup_{\substack{\|\bw_j\|_p\leq W\\ \|\bu\|_1\leq \Lambda}}\sum_{t=1}^{n_\cC}\left|\sum_{i\in \cC_t}\sigma_i\right|\Big]&(\bs_i\text{ constraint})\nonumber\\
&=\frac {\epsilon W\Lambda} m\max(1,d^{1-\frac 1p -\frac 1r})\E_\bsig \Big[\sup_{\substack{\|\bw_j\|_p\leq W\\ \|\bu\|_1\leq \Lambda}}\sum_{t=1}^{n_\cC}\left|\sum_{i\in \cC_t}\sigma_i\right|\Big]&(\text{Lemma~\ref{lemma:norm_ratio}})\nonumber\\
&\leq \frac {\epsilon W\Lambda} m\max(1,d^{1-\frac 1p -\frac 1r})\E_\bsig \Big[\sum_\cC\sum_{t=1}^{n_\cC}\left|\sum_{i\in \cC_t}\sigma_i\right|\Big]&(\text{sum over all classes})\nonumber\\
&=\frac {\epsilon W\Lambda} m\max(1,d^{1-\frac 1p -\frac 1r})\sum_\cC\sum_{t=1}^{n_\cC}\E_\bsig \Big[\left|\sum_{i\in \cC_t}\sigma_i\right|\Big] \nonumber .
\end{align}
By Jensen's inequality, we have
$$
\E_\bsig \Big[|\sum_{i\in \cC_t}\sigma_i|\Big] \leq \sqrt{|\cC_t|}.
$$
Substituting this bound above we get that
\begin{align}
    \E_\bsig\Bigg[\sup_{\substack{\|\bw_j\|_p\leq W\\ \|\bu\|_1\leq \Lambda}}\frac 1m \sum_{t=1}^{n_\cC} \sum_{i\in \cC_t} \sigma_i y_i\sum_{j\in P_t}\e u_j\bw_j\cdot \bs_t\Bigg] &\leq \frac {\epsilon W\Lambda} m\max(1,d^{1-\frac 1p -\frac 1r})\sum_\cC\sum_{t=1}^{n_\cC} \sqrt{|\cC_t|} \nonumber \\
    &\leq \frac{\epsilon \Lambda W}{\sqrt{m}}  \max(1,d^{1-\frac 1p -\frac 1r}) |\mathcal{C}_{\cS}^*| \sqrt{\Pi_\cS^*} \label{eq:nn-dim-ind-2}.
\end{align}
Combining \eqref{eq:nn-dim-ind-1} and \eqref{eq:nn-dim-ind-2} completes the proof.
\end{proof}

We would like to point out that in the above analysis one can replace the dependence on $\|\bX\|_{\infty, p^*}$ with a dependence on $\|\bX\|_{2,p^*}$ at the expense of a slower rate of convergence (in terms of $m$). In order to do this we use Proposition~\ref{prop:norm_ratio} to bound the right hand side of \eqref{eq:X_t_top} as:
\begin{align*}
    \|\bX^\top_t\|_{2,p^*} \leq \max (1, m^{\frac{1}{p^*} - \frac 1 2}) \|\bX\|_{p^*,2}.
\end{align*}
Substituting the above bound into the analysis we get the following corollary.
\begin{corollary}
Consider the family of functions $\cG^n_p$ with $p \in [1,\infty)$, activation function $\rho(z)=(z)_+$,
and perturbations in $r$-norm for $1<r<\infty$. Assume that for our sample $\|\bx_i\|_r\geq \e$.
Then, the following upper bound on the Rademacher complexity holds:
\[
\wt \R_\cS(\cG_p^n)
\!\leq\! \bigg[\! \frac{W\Lambda \max(1,d^{1-\frac 1p-\frac 1r}) \Big(K(p,d)\max (1, m^{\frac{1}{p^*} - \frac 1 2}) \|\bX\|_{p^*,2}+\e \Big)}{\sqrt m} \!\bigg] C^*_\cS {\Pi^*_\cS},
\]
\end{corollary}

\subsection{Bounding ${\Pi_\cS^*}$.} 

Notice that a key data dependent quantity that controls the Rademacher complexity bound in the previous analysis is $\Pi_\cS^*$, i.e., the maximum number of partitions that $\cS$ can induce on the weights $\bw_1, \dots, \bw_k$. As mentioned in Section~\ref{sec:dim-independent-neural-nets} our notion of $\epsilon$-adversarial shattering provides a general way to bound ${\Pi_\cS^*}$. We restate the definition of $\epsilon$-adversarial shattering here and then discuss its implications.
\begin{definition}
Fix the sample $\cS = ((\bx_1, y_1)\ldots (\bx_m, y_m))$ and $(\bw_1, \ldots, \bw_n)$.
Let $\bs_i = \argmin_{\|\bs\|_r\leq 1}y_i\sum_{j=1}^n
u_j(\bw_j\cdot (\bx_i+\e \bs))_+$, and define the following three sets:
\begin{align*}
    P_i=\{j\colon \bw_j\cdot (\bx+\e \bs_i)>0\}\\
    Z_i=\{j\colon \bw_j\cdot (\bx+\e \bs_i)=0\}\\
    N_i=\{j\colon \bw_j\cdot (\bx+\e \bs_i)<0\}.
\end{align*}
Let $\Pi_\cS(\bW)$ be the number of distinct $(P_i,Z_i,N_i)$s that are induced by $\cS$, where $\bW$ is a matrix that admits the $\bw_j$s as columns. We call $\Pi_\cS(\bW)$ the \emph{$\e$-adversarial
  growth function}. We say that $\bW$ is \emph{$\e$-adversarially
  shattered} if every $P\subset [n]$ is possible.
\end{definition}
We will further study the above notion of $\epsilon$-adversarial shattering to bound $\Pi_\cS^*$ under assumptions on the weight matrix $\bW$. In particular, we will be interested in vectors $\bw_1, \dots, \bw_n$ such that for all $i \in [n]$, the set $Z_i$ is empty. In this case we say that $\bW$ is $\epsilon$-adversarially shattered if every partition of the weights into sets $P_i,N_i$ is possible. For this setting, we state below a lemma that is analogous to Sauer's lemma in statistical learning theory \citep{sauer1972density, shelah1972combinatorial} and helps us bound the $\epsilon$-adversarial growth function $\Pi_\cS(\bW)$.
\begin{lemma}
\label{lem:sauer_adv}
Fix an integer $t \geq 1$. Fix a sample $\cS = ((\bx_1, y_1)\ldots (\bx_m, y_m))$ and weights $\bw_1, \dots, \bw_n$ such that for all $i \in [n]$, $Z_i = \emptyset$, and no subset of the weights of size more than $t$ can be $\epsilon$-adversarially shattered by $\cS$. Then it holds that
\begin{align}
\label{eq:sauer_adv}
    \Pi_\cS(\bW) \leq \sum_{i=0}^t {n \choose i}.
\end{align}
\end{lemma}
\begin{proof}
The proof is similar to the proof of Sauer's lemma \citep{sauer1972density, shelah1972combinatorial} and use an induction on $n+t$.

\noindent \textbf{Base Case.} We first show that for $n=0$ and any $t$,
$$
\Pi_\cS(\bW) \leq \sum_{i=0}^t {0 \choose i} = 1.
$$
This easily follows since if $n=0$, there is no set to shatter. Next, we show that for $t=0$ and any $n$,
$$
\Pi_\cS(\bW) \leq \sum_{i=0}^0 {n \choose i} = 1.
$$
The above holds since if no set of size one can be shattered, then all the points in $\cS$ fall in a single part of the partition.

\noindent \textbf{Inductive Step.} Let $n+t = k$ and assume that \eqref{eq:sauer_adv} holds for all $n,t$ with $n+t < k$. Notice that $\Pi_\cS(\bW)$ is simply the maximum number of labelings of $W$ that can be induced by $\cS$. Let $A$ be the set of all such labelings and let $A'$ be the smallest subset of $A$ that induces the maximal number of different labelings on $\bw_2, \dots, \bw_n$. Notice that $A'$ cannot shatter more than $t$ of the weights in $\bw_2, \dots, \bw_n$. Furthermore, $A \setminus A'$ cannot shatter more than $t-1$ of the weights, since any labeling in $A \setminus A'$ has a corresponding labeling in $A$ with opposite label on $\bw_1$. Hence, if $A \setminus A'$ shatters more than $t-1$ of the weights in $\bw_2, \dots, \bw_n$ then we get that $A$ shatters more than $t$ of the weights in $\bw_1, \dots, \bw_n$. Finally, using the induction hypothesis we get that
\begin{align*}
    \Pi_\cS(\bW) &= |A|\\
    &= |A'| + |A \setminus A'|\\
    &\leq \sum_{i=0}^t {n-1 \choose i} + \sum_{i=0}^{t-1} {n-1 \choose i}\\
    &= \sum_{i=0}^t {n \choose i}.
\end{align*}
\end{proof}
Finally, we end the section by demonstrating that the notion of $\epsilon$-adversarial shattering can lead to dimension independent bounds on $\Pi^*_\cS$ under certain assumptions. We believe that this notion warrants further investigation and is key in deriving dimension independent bounds for more general setting. Below we analyze a special case of orthogonal vectors.
\begin{lemma}
Fix $p > 1$. Let $\cS = ((\bx_1, y_1)\ldots (\bx_m, y_m))$ be a sample and $\bw_1, \dots \bw_t$ be a set of weight vectors. Let $\bW$ be the matrix with $\bw_i$s as columns. Furthermore, we make the following assumptions
\begin{enumerate}
\item $\|\bw_j\|^2 \geq w^2_{\text{min}}$ for all $j \in [t]$.
\item $\bw_j \cdot \bw_k = 0$ for all $j \neq k$.
\item $\|\bW^\top\|_{2,p^*} \leq \tau$.
\item $u_j = 1$.
\end{enumerate}
If $\cS$ $\epsilon$-adversarially shatters $\bw_1, \dots \bw_t$ with perturbations measured in $r=2$ norm then it holds that
\begin{align*}
    t \leq \frac{4\tau^2 c^2_2(p^*) \|\bX\|^2_{p, \infty}}{\epsilon^2 w^2_{\text{min}}},
\end{align*}
where the constant $c_2(p^*)$~(as in Lemma~\ref{lemma:f2_bound}) is defined as,
$$
c_2(p^*)\colon = \sqrt 2 \big(\frac{\Gamma(\frac{\du p + 1} 2 )}{\sqrt \pi}\big)^\frac 1 {\du p}.
$$
\end{lemma}
\begin{proof}
For orthogonal $\bw_j$'s, Theorem~\ref{th:opt_char_r_2} implies that $Z_i=\emptyset$. Thus, the optimal perturbation is characterized by 
\[\bs_i^*=-\frac{\sum_{j\in P_i} \bw_j}{\|\sum_{j\in P} \bw_i\|_2}\] In the following, it will be more convenient to work with the negative of this quantity, so we define

\begin{align*}
    \bs_i &=-\bs_i^*= \frac{\sum_{j \in P_i} \bw_j }{\|\sum_{j \in P_i} \bw_j\|_2}.
\end{align*}
For a given shattering $P_i, N_i$ by an example $\bx_i$ the following holds:
\begin{align}
    \forall j \in P_i, (\bw_j \cdot \bx_i - \epsilon \bw_j \cdot \bs_i) &> 0 \label{eq:shatter_plus}\\
    \forall j \in N_i, (\bw_j \cdot \bx_i - \epsilon \bw_j \cdot \bs_i) & < 0. \label{eq:shatter_minus}
\end{align}
Next, we define $\bW^+$ and $\bW^{-}$ as follows:
\begin{align*}
    \bW^+ &= \sum_{j \in P_i} \bw_j\\
    \bW^- &= \sum_{j \in N_i} \bw_j.
\end{align*}
Furthermore, let $\Delta \bW = \bW^+ - \bW^-$. Then summing over the inequalities in \eqref{eq:shatter_plus} and \eqref{eq:shatter_minus} we can write:
\begin{align*}
    \Delta \bW \cdot \bx_i &> \epsilon \Delta \bW \cdot \bs_i\\
    &= \epsilon \frac{\Delta \bW \cdot \bW^+}{\|\bW^+\|_2}
\end{align*}
Using the fact that $|\Delta \bW \cdot \bx_i| \leq \|\Delta \bW\|_{p^*} \|\bX\|_{p, \infty}$ we can write:
\begin{align}
    \|\bW^+\|_2 \|\bW\|_{p^*} \|\bX\|_{p, \infty} > \epsilon \Delta \bW \cdot \bW^+. \label{eq:shatter_per_point}
\end{align}
Since $\cS$ $\epsilon$-adversarially shatters $\bW$, \eqref{eq:shatter_per_point} must hold for every partition $P_i, N_i$, and hence must hold in expectation over the random partition as well. Hence, introducing Rademacher random variables $\sigma_1, \dots, \sigma_t$ we can write:
\begin{align}
    \E_{\bsigma} \big[\|\bW^+\|_2 \|\Delta \bW\|_{p^*} \|\bX\|_{p, \infty} \big] > \epsilon \E_{\bsigma} \big[\Delta \bW \cdot \bW^+ \big], \label{eq:shatter_exp}
\end{align}
where $\bW^+ = \sum_{j=1}^t 1_{\sigma_j > 0} \bw_j$ and $\Delta \bW = \sum_{j=1}^t \sigma_j \bw_j$. We bound the right-hand side in \eqref{eq:shatter_exp} above as
\begin{align}
    \epsilon \E_{\bsigma} \big[\Delta \bW \cdot \bW^+ \big] &= \epsilon \E_{\bsigma} \big[ \big(\sum_{j=1}^t \sigma_j \bw_j \big) \big(\sum_{j=1}^t \sigma_j 1_{\sigma_j > 0} \bw_j \big) \big]\\
    &= \epsilon \sum_{j,k=1}^t \E[1_{\sigma_j > 0} \sigma_k] \bw_j \cdot \bw_k \nonumber \\
    &= \epsilon \Big( \sum_{j \neq k} \E[1_{\sigma_j > 0}]\E[\sigma_k] \bw_j \cdot \bw_k + \sum_{j=1}^t \E[1_{\sigma_j > 0}] \bw_j \cdot \bw_j \Big) \nonumber \\
    &= \frac{\epsilon}{2} \sum_{j=1}^t \|\bw_j\|^2 \label{eq:shatter_rhs_exp}.
\end{align}
Next, using Cauchy-Schwarz inequality we upper bound the left hand side of \eqref{eq:shatter_exp} as
\begin{align}
    \E_{\bsigma} \big[\|\bW^+\|_2 \|\Delta \bW\|_{p^*} \|\bX\|_{p, \infty} \big] &\leq \sqrt{\E_\bsigma[\|\bW^+\|^2_2]} \sqrt{\E_\bsigma[\|\Delta \bW\|^2_{p^*}]}\|\bX\|_{p, \infty} \nonumber \\
    &\leq \sqrt{\sum_{j=1}^t \E[1_{\sigma_j > 0}]\|\bw_j\|^2}\sqrt{\E_\bsigma[\|\Delta \bW\|^2_{p^*}]}\|\bX\|_{p, \infty} \,\, \text{[Using orthogonality of the $\bw_j$ vectors.]} \nonumber \\
    &=\sqrt{\frac 1 2\sum_{j=1}^t \|\bw_j\|^2}\sqrt{\E_\bsigma[\|\Delta \bW\|^2_{p^*}]}\|\bX\|_{p, \infty} \label{eq:shatter_lhs_1}.
\end{align}
Furthermore, since $p^* > 1$, using the analysis in Section~\ref{app:linear_rc} and the Khintchine-Kahane inequality \citep{Haagerup1981}:
\begin{align}
    \E_\bsigma[\|\Delta \bW\|^2_{p^*}] &\leq 2\E_\bsigma[\|\Delta \bW\|_{p^*}]^2 \nonumber \\
    &= 2\E_\bsigma[\|\sum_{j=1}^t \sigma_j \bw_j \|_{p^*}]^2 \nonumber\\
    &\leq 2 c^2_2(p^*)\|\bW^\top\|^2_{2,p^*} \nonumber\\
    &\leq 2 c^2_2(p^*)\tau^2. \label{eq:shatter_lhs_2}
\end{align}
Combining \eqref{eq:shatter_rhs_exp}, \eqref{eq:shatter_lhs_1} and \eqref{eq:shatter_lhs_2} we can write:
\begin{align*}
    \epsilon \sqrt{\frac 1 2 \sum_{j=1}^t \|\bw_j\|^2} &< \sqrt{2} c_2(p^*) \tau \|\bX\|_{p, \infty}.
\end{align*}
From our assumption we also have that $\|\bw_j\|^2 \geq w^2_{\text{min}}$ for all $j \in [t]$. Substituting above we get
\begin{align*}
    \epsilon \cdot w_{\text{min}}\sqrt{\frac t 2} < \sqrt{2} c_2(p^*) \tau \|\bX\|_{p, \infty}.
\end{align*}
Rearranging, we get that
\begin{align*}
    t \leq \frac{4 c^2_2(p^*) \tau^2 \|\bX\|^2_{p, \infty}}{\epsilon^2 w^2_{\text{min}}}.
\end{align*}
\end{proof}

\end{document}